\documentclass[]{TEAI}
\usepackage{amsmath} 
\usepackage{natbib}
\usepackage{graphicx}
\usepackage{subcaption} 

\usepackage[toc,page,header]{appendix}
\usepackage[utf8]{inputenc} 
\usepackage[T1]{fontenc}    
\usepackage{hyperref}       
\usepackage{url}            
\usepackage{booktabs}       
\usepackage{amsfonts}       
\usepackage{nicefrac}       
\usepackage{microtype}      
\usepackage{wrapfig}

\usepackage{amssymb}  
\usepackage{fontawesome}  
\usepackage{url}  

\usepackage{titletoc}

\usepackage{tikz}  
\usepackage{comment} 
\usepackage{tabularx}  
\usepackage{booktabs}  

\usepackage{minitoc}

\usepackage{booktabs}
\usepackage{array}
\usepackage{etoolbox}

\definecolor{lightblue}{RGB}{200, 230, 255}  
\definecolor{headerblue}{RGB}{150, 200, 255} 

\usepackage{pgfplots}
\usepackage[utf8]{inputenc} 
\usepackage[T1]{fontenc}    
\usepackage{hyperref}       
\usepackage{url}            
\usepackage{booktabs}       
\usepackage{amsfonts}       
\usepackage{nicefrac}       
\usepackage{microtype}      
\usepackage{xcolor}         
\usepackage{graphicx}
\usepackage{float}
\usepackage{comment}
\usepackage{multirow} 
\usepackage{amsmath} 
\usepackage{makecell} 
\usepackage{siunitx}  
\usepackage{tikz}
\usepackage{pgf-pie} 
\usepackage{subcaption}
\usepackage{wrapfig}
\usepackage[export]{adjustbox}

\usepackage{ragged2e}      
\usepackage{tabularx}       
\usepackage{array}          
\usepackage{caption}        
\usepackage{enumitem}
\usepackage{pifont}
\usepackage[hang,flushmargin]{footmisc} 

\usepackage{tcolorbox}
\usepackage{tcolorbox}    
\tcbuselibrary{breakable}  
\tcbuselibrary{skins}      
\usepackage{tabularx}
\usepackage{listings}

\usepackage{amsthm}
\newtheorem{theorem}{Theorem}
\theoremstyle{definition}
\newtheorem{definition}{Definition}

\theoremstyle{plain}
\newtheorem{lemma}{Lemma}
\newtheorem{proposition}{Proposition}

\usepackage{algorithm}
\usepackage{algorithmic}

\title{\textsc{ArcFlow}: Unleashing 2-Step Text-to-Image Generation via High-Precision Non-Linear Flow Distillation}

\author{
    Zihan Yang\textsuperscript{1,*},
    Shuyuan Tu\textsuperscript{1,*},
    Licheng Zhang\textsuperscript{1},  
    Qi Dai\textsuperscript{2},  
    Yu-Gang Jiang\textsuperscript{1},  
    Zuxuan Wu\textsuperscript{1}
}

\affiliation[1]{\mbox{Fudan University}} 
\affiliation[2]{\mbox{Microsoft Research Asia}}
    
\abstract{
\begin{abstract}

Diffusion models have achieved remarkable generation quality, but they suffer from significant inference cost due to their reliance on multiple sequential denoising steps, motivating recent efforts to distill this inference process into a few-step regime. However, existing distillation methods typically approximate the teacher trajectory by using linear shortcuts, which makes it difficult to match its constantly changing tangent directions as velocities evolve across timesteps, thereby leading to quality degradation. 
To address this limitation, we propose ArcFlow, a few-step distillation framework that explicitly employs non-linear flow trajectories to approximate pre-trained teacher trajectories. 
Concretely, ArcFlow parameterizes the velocity field underlying the inference trajectory as a mixture of continuous momentum processes. This enables ArcFlow to capture velocity evolution and extrapolate coherent velocities to form a continuous non-linear trajectory within each denoising step. Importantly, this parameterization admits an analytical integration of this non-linear trajectory, which circumvents numerical discretization errors and results in high-precision approximation of the teacher trajectory.
To train this parameterization into a few-step generator, we implement ArcFlow via trajectory distillation on pre-trained teacher models using lightweight adapters. This strategy ensures fast, stable convergence while preserving generative diversity and quality. 
Built on large-scale models (Qwen-Image-20B and FLUX.1-dev), ArcFlow only fine-tunes on less than $5\%$ of original parameters and achieves a $40 \times$ speedup with 2 NFEs over the original multi-step teachers without significant quality degradation. Experiments on benchmarks show the effectiveness of ArcFlow both qualitatively and quantitatively.
\end{abstract}
}

\checkdata[Website]{\url{https://github.com/pnotp/ArcFlow}}

\begin{document}
\maketitle
\renewcommand{\thefootnote}{}
\footnotetext{$^*$Equal Contribution.}
\renewcommand{\thefootnote}{\arabic{footnote}}

\vspace{-1.5em}

\begin{figure}[h]
    \centering
    \includegraphics[width=1\textwidth]{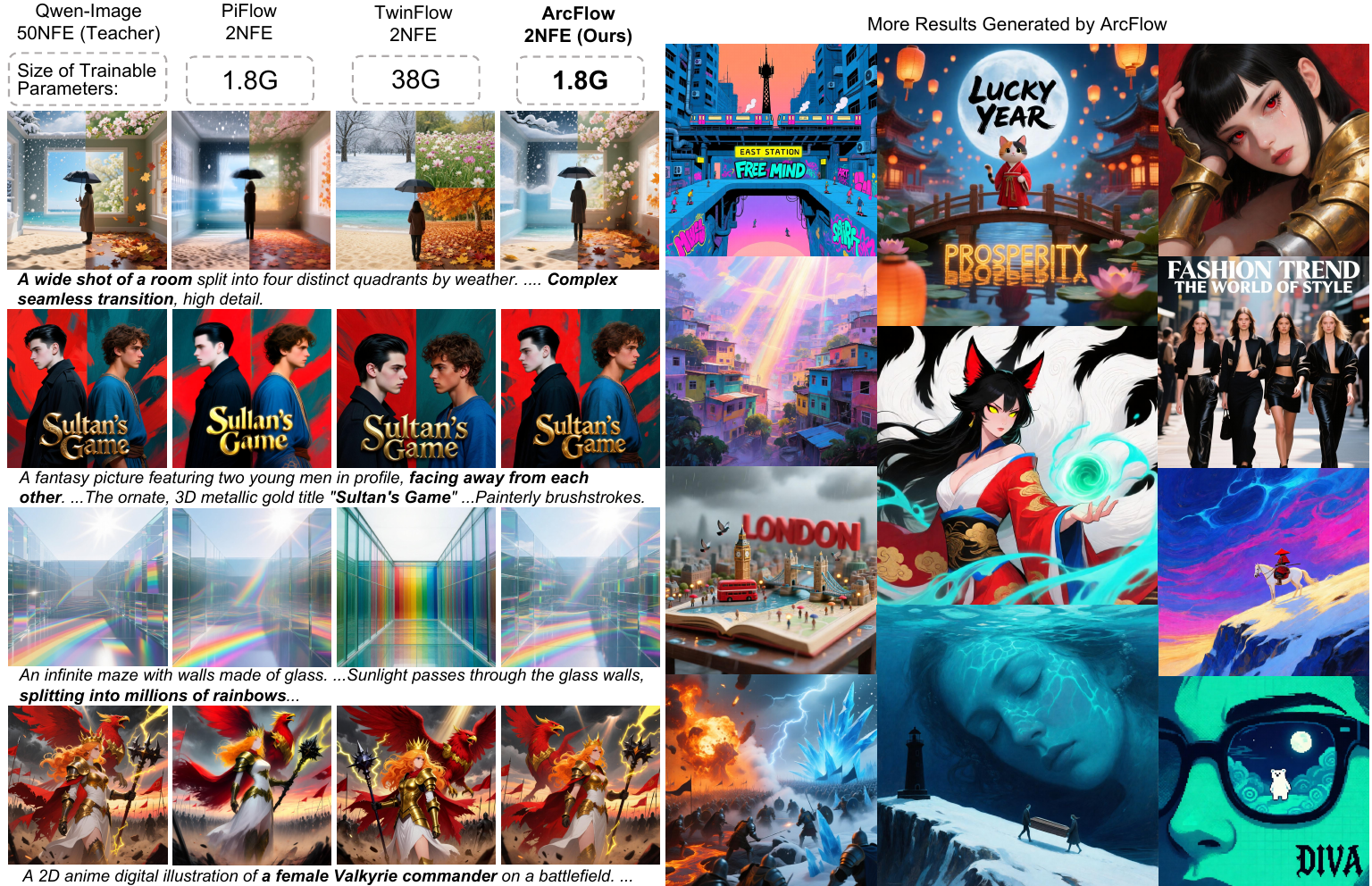}
    \caption{Comparisons between images generated by ArcFlow and other state-of-the-art distillation methods based on Qwen-Image-20B, demonstrating the power of ArcFlow for few-step high-fidelity generation while maintaining remarkable parameter efficiency. 
    }
    \label{fig:cover}
\end{figure}

\section{Introduction}

Diffusion and flow matching models have emerged as the dominant paradigms for high-fidelity visual generation ~\cite{ho2020denoisingdiffusionprobabilisticmodels,lipman2023flowmatchinggenerativemodeling,albergo2023buildingnormalizingflowsstochastic, tu2024motioneditor, tu2024motionfollower, tu2025stableanimator, tu2025stableanimator++, tu2025stableavatar, tu2025flashportrait, tu2023implicit}. Despite their impressive capabilities, they rely on iterative differential equation solvers, typically requiring 40 to 100 denoising steps to traverse the trajectory from noise to data, making them impractical for real-time applications. Therefore, accelerating sampling without compromising quality remains a critical challenge.

To address this issue, recent research has explored various paradigms to distill a pre-trained teacher model into a few-step student generator. Different methods range from progressive distillation~\cite{salimans2022progressivedistillationfastsampling,meng2023distillationguideddiffusionmodels} to consistency-based approaches~\cite{song2023consistencymodels,luo2023latentconsistencymodelssynthesizing}, and distribution matching ~\cite{sauer2023adversarialdiffusiondistillation,yin2024onestepdiffusiondistributionmatching,cheng2025twinflowrealizingonestepgeneration} that employ adversarial or divergence losses to align distributions.

However, their essence still lies in approximating the trajectory from the teacher generation process ($40\sim100$ steps), whose tangent directions vary over multiple timesteps, via a linear shortcut under very few steps ($2\sim4$ steps). This enforces the students to implicitly learn such tangent variation with linear trajectories, leading to geometric mismatch.

In light of this, we propose ArcFlow, a few-step distillation framework that introduces explicit non-linear flow trajectories via velocity parameterization to approximate the flow trajectories from a pre-trained teacher model. Since the flow trajectory is equivalent to how the velocity evolves across timesteps, we utilize the notion of momentum in physics~\cite{goldstein2002classical} to describe this evolution, where the overall trajectory is determined only by the initial velocity and a momentum factor. 
Consequently, we parameterize the velocity field as a weighted mixture of continuous momentum processes. 
By harnessing the continuity of this momentum process over adjacent timesteps, the model can extrapolate coherent velocities through the parameterization, thus efficiently constructing the non-linear trajectories based on the predicted velocity shifts. 
Notably, our parameterization admits a closed-form analytical solution to the flow ODE~\cite{song2021scorebasedgenerativemodelingstochastic}, enabling direct computation of the terminal state in a single forward pass. 
This ensures the predicted velocity evolution is applied accurately across timesteps within the interval, rather than being approximated by a linear discrete update, thereby ensuring high-precision flow distillation.

By tackling the geometric mismatch, ArcFlow ensures that the trajectory of the student naturally aligns with the teacher’s inherent tangent variation. This alignment fundamentally simplifies the distillation task, enabling parameter-efficient training. Unlike prior methods requiring full-model training, ArcFlow achieves state-of-the-art results by fine-tuning only lightweight LoRA adapters and the output head. 

As illustrated in \cref{fig:cover}, with only 2 NFEs, ArcFlow achieves high-fidelity generation comparable to the teacher Qwen-Image-20B, surpassing the 2-step generation quality of pi-Flow~\cite{piflow} and TwinFlow~\cite{cheng2025twinflowrealizingonestepgeneration} while utilizing very few trainable parameters.
Furthermore, the convergence analysis (\cref{fig:convergence}) highlights that the training of ArcFlow yields significantly faster convergence and superior stability, validating the effectiveness of the alignment with non-linear trajectory that efficiently eliminates the geometric optimization bottleneck.

\vspace{-8pt}
\begin{figure}[h]  
    \begin{minipage}{0.6\columnwidth}
         Our main contributions are as follows: 
        (1) We propose ArcFlow, the first distillation framework to explicitly construct a non-linear flow trajectory to approximate the teacher trajectory. We parameterize the velocity as a continuous momentum mixture, whose analytic solution for trajectory integration ensures high-precision alignment with the teacher.
        (2) We introduce an analytic trajectory solver for ArcFlow, which enables an efficient objective for distillation. It simplifies the training process, enabling parameter-efficient adaptation and fast convergence.
        (3) Evaluations on benchmark datasets demonstrate superior robustness of ArcFlow, achieving SOTA across diverse backbones. It achieves a $40\times$ inference speedup over the teacher and at most $4\times$ faster training convergence than prior methods, while fine-tuning only less than $5\%$ of the original parameters.
    \end{minipage}
    \hfill 
    \begin{minipage}{0.38\columnwidth}
        \centering
        \vspace{0.3cm}
        \includegraphics[width=\linewidth]{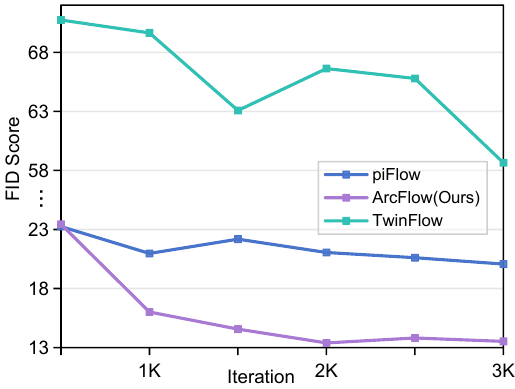}
        \vspace{-0.3cm}
        \caption{Comparison of FID scores across training iterations for different methods. ArcFlow achieves superior convergence speed.}
        \label{fig:convergence}
    \end{minipage}
\end{figure}

\section{Related Work}
\noindent\textbf{Text-to-image generation}
Diffusion~\cite{ho2020denoisingdiffusionprobabilisticmodels} and flow matching models~\cite{lipman2023flowmatchinggenerativemodeling} have emerged as the standard for high-resolution visual synthesis. Recent scaling efforts, such as Stable Diffusion 3~\cite{esser2024scalingrectifiedflowtransformers}, FLUX~\cite{flux2024,flux-2-2025}, and Qwen-Image~\cite{wu2025qwenimagetechnicalreport}, leverage Transformer to achieve exceptional performance. However, these models fundamentally rely on integrating probability flow ODEs via iterative numerical solvers. They necessitates 40 to 100 function evaluations (NFEs), creating a significant latency bottleneck that hinders real-time deployment and necessitates acceleration.

\noindent\textbf{Few-step Image Generation}
Accelerating the inference of diffusion models has become a critical topic, aiming to achieve high-fidelity synthesis with few function evaluations (NFEs). To this end, knowledge distillation has emerged as a dominant paradigm, where a student model is trained to approximate the complex sampling trajectory of a pre-trained teacher. One line of work focuses on trajectory simplification, such as Progressive Distillation ~\cite{salimans2022progressivedistillationfastsampling, meng2023distillationguideddiffusionmodels} and Rectified Flow ~\cite{liu2022flowstraightfastlearning}, attempting to reduce NFEs by iteratively straightening the flow. However, they struggle to eliminate discretization errors in the few-step regime. 
Consistency Models~\cite{song2023consistencymodels,luo2023latentconsistencymodelssynthesizing} map points directly to the data via self-consistency constraints, but they often require computationally expensive Jacobian-vector product calculations to maintain convergence stability~\cite{geng2025meanflowsonestepgenerative}.

To further push limits to 1-4 steps, VSD~\cite{wang2023prolificdreamerhighfidelitydiversetextto3d} and DMD~\cite{yin2024onestepdiffusiondistributionmatching} introduce discriminator-based losses, and TwinFlow~\cite{cheng2025twinflowrealizingonestepgeneration} uses a self-adversarial objective. While these improve visual sharpness, the reliance on adversarial objectives and unstable training leads to mode collapse and high memory overhead. 
Recent attempts~\cite{gmflow,piflow} approximate evolution of velocities via Gaussian mixtures. However, their probabilistic approximations lack precision at lower NFEs (2 steps). 
By contrast, ArcFlow utilizes an analytic momentum solver to achieve precise, stable, and parameter-efficient distillation.
\begin{figure*}[t!]
    \centering
    \includegraphics[width=1.0\linewidth]{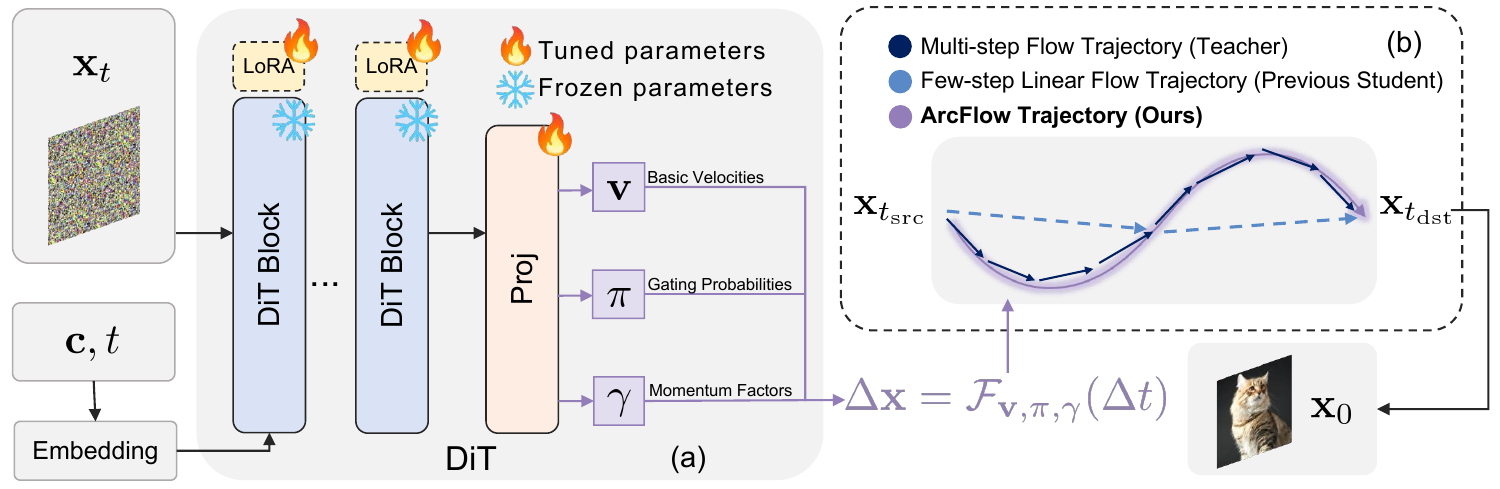}
    \caption{ArcFlow Framework. (a) The forward pipeline of ArcFlow. Given an input $x_t$, the condition $c$ and timestep $t$, a DiT backbone with three projection heads predicts the parameters $v$,  $\omega$, $\gamma$ across $K$ dynamic modes, which respectively denote the mode-specific velocities, momentum factors, and the gating probabilities used to reconstruct the teacher velocity field.  (b) A comparison of flow trajectories produced by the multi-step teacher model, the few-step linear student model, and our ArcFlow.}
    \label{fig:framework}
\end{figure*}

\section{Method}
Pre-trained diffusion models follow PF-ODE integration trajectories with constantly varying tangents~\cite{song2021scorebasedgenerativemodelingstochastic}, whereas existing distillation methods~\cite{song2023consistencymodels,yin2024onestepdiffusiondistributionmatching,cheng2025twinflowrealizingonestepgeneration} approximate them using linear shortcuts, resulting in geometric mismatch.

In light of this, we propose ArcFlow, a text-to-image distillation framework that utilizes the notion of momentum process~\cite{goldstein2002classical} to construct non-linear flow trajectories across long timestep intervals. 
In this section, we first formalize the momentum-based parameterization (\cref{method1}), derive the analytical trajectory integration solver (\cref{method2}), and detail the trajectory distillation strategy to train ArcFlow (\cref{method3}). 
The condition variable (text prompt) is omitted from subsequent descriptions for brevity.
Since ArcFlow mainly relies on learning from a pre-trained teacher, we define the velocity field from the frozen teacher as ground truth for the subsequent analysis.

\subsection{Momentum Parameterization of Probability Flow}
\label{method1}

The Probability Flow ODE framework~\cite{song2021scorebasedgenerativemodelingstochastic} reveals that the diffusion process follows a continuous trajectory, where the denoising velocities are strongly correlated between adjacent timesteps. However, standard numerical solvers (e.g., Euler method~\cite{butcher2016numerical}) appoximate the integration process by taking discrete steps independently without considering their associations across timesteps. Thus, we argue that the standard multi-step sampling, which re-evaluates the network repeatedly to traverse this smooth trajectory, suffers from severe redundancy.

To explicitly exploit this inherent continuous evolution of the velocity field across timesteps, we introduce the notion of momentum in physics~\cite{goldstein2002classical}, to parameterize such properties. 
Specifically, let $\mathbf{v}\left(\mathbf{x}_t, t\right)$ denote the velocity field at timestep $t \in [0,1]$. The relationship between velocities at adjacent timesteps should follow a momentum transmission law parameterized by a factor $\gamma$. 
It implies that the velocity transfer as $\mathbf{v}\left(\mathbf{x}_t,t\right)=\mathbf{v}\left(\mathbf{x}_{t+\Delta t},t+\Delta t\right) \cdot \gamma^{\Delta t}$ from $t+\Delta t$ to $t$.

Recursively apply the above formula from a starting timestep $t_{s}$ to any ending timestep $t \in [0, t_{s})$, the velocity evolution is derived as follows:
\begin{equation}\small
    \mathbb{E}\left[ \mathbf{v}\left(\mathbf{x}_{t}, t\right) \mid \mathbf{v}\left(\mathbf{x}_{t_{s}}, t_{s}\right), \gamma \right] = \mathbf{v}\left(\mathbf{x}_{t_{s}}, t_{s}\right) \cdot \gamma^{t_{s}-t},
    \label{eq:momentum_recurrence}
\end{equation}
where $\gamma \in \mathbb{R}^+$.
Based on Eq.~\eqref{eq:momentum_recurrence}, given the initial velocity $\mathbf{v}\left(\mathbf{x}_{t_{s}},t_{s}\right)$ , $\mathbf{v}\left(\mathbf{x}_{t},t\right)$ at any timestep $t \in \left[0, t_{s}\right)$ can be extrapolated directly. Thus, our momentum parameterization allows flow matching to analytically predict the velocity at every timestep after only a single NFE, reaching $\mathbf{x}_{t}$ directly.

While momentum helps approximate velocity evolution, a single momentum factor $\gamma$ is insufficient to capture the hierarchical frequency dynamics in image generation. Empirical studies~\cite{choi2022perceptionprioritizedtrainingdiffusion} show that different frequency components evolve at distinct rates during denoising, implying that the corresponding velocity field $\mathbf{v}\left(\mathbf{x}_t,t\right)$ inherently consists of multiple evolution modes with different decay rates. 
To model such dynamics, as shown in \cref{fig:framework}(a), we formulate $\mathbf{v}\left(\mathbf{x}_t,t\right)$ as a probabilistic mixture of $K$ distinct momentum modes. 
Specifically, we decompose the velocity field into K different modes indexed by $z \in [1, ..., K]$ , and then derive the overall velocity field $\mathbf{v}_\theta\left(\mathbf{x}_t,t\right)$, optimized by parameter $\theta$:

\begin{equation}\small
    \begin{aligned}
    \mathbf{v}_\theta\left(\mathbf{x}_t, t\right)
    &= \mathbb{E}_{z \sim p_\theta\left(z|\mathbf{x}_t\right)}
    \left[ \mathbf{v}\left(\mathbf{x}_t, t \mid z\right) \right] \\
    &= {\sum}_{k=1}^K 
    \underbrace{p_\theta\left(z=k|\mathbf{x}_t\right)}_{\pi_k\left(\mathbf{x}_t\right)}
    \cdot
    \underbrace{\mathbf{v}_k\left(\mathbf{x}_t\right) \cdot \gamma_k\left(\mathbf{x}_t\right)^{1-t}}_{\text{Mode-specific Dynamics}},
    \end{aligned}
    \label{eq:mixture_expectation}
\end{equation}

where $\pi_k(\mathbf{x}_t) \in [0, 1]$ refers to the gating probability predicted by the parameter $\theta$, subject to $\sum \pi_k = 1$. $\mathbf{v}_k(\mathbf{x}_t) \in \mathbb{R}^D$ and $\gamma_k(\mathbf{x}_t) \in \mathbb{R}^+$ are the predicted basic velocity and momentum factor for the $k$-th mode.
Consequently, ArcFlow divides the trajectory into several mode-specific sub-trajectories, enabling each to be learned in a more targeted manner, thereby improving overall learning efficiency.

To further prove the rationality of our parameterization, we introduce a theorem showing that 
Eq.~\eqref{eq:mixture_expectation} with $K$ dynamic modes theoretically admits a parameter setting that perfectly fits the sampled trajectory at $N \leq K$ distinct timesteps. 

\begin{theorem}
\label{thm:velocity_alignment}
Consider the velocity field predicted by ArcFlow at any sampled latent $\mathbf{y}$ and timestep $t$, parameterized as $\mathbf{v}_{\theta}(\mathbf{y}, t)
= \sum_{k=1}^K \pi_k(\mathbf{y})\, \mathbf{v}_k(\mathbf{y})\, \gamma_k(\mathbf{y})^{1-t}$ according to Eq.~\eqref{eq:mixture_expectation}.
Let $\mathbf{u}^*(\mathbf{y}, t)$ denote the ground-truth velocity field, observed at $N$ distinct timesteps
$\mathcal{T} = \{t_1, \dots, t_N\} \subset (0,1]$.
If the number of modes satisfies $K \ge N$, then there exists a parameter configuration
$\theta = \{\pi_k, \mathbf{v}_k, \gamma_k\}_{k=1}^K$:
\begin{equation}\small
    \mathbf{v}_{\theta}(\mathbf{y}, t_n) = \mathbf{u}^*(\mathbf{y}, t_n),
    \quad \forall t_n \in \mathcal{T},
\end{equation}
\end{theorem}

We prove the theorem in \cref{proof:velocity_alignment}. 
The theoretical result validates that the momentum-based parameterization is capable of approximating the ground-truth velocity field in a non-linear way, ensuring high-precision distillation.


\subsection{Analytic ODE Solvers}
\label{method2}
As described above, ArcFlow parameterizes the velocity field as a mixture of momentum modes (Eq.~\eqref{eq:mixture_expectation}), which is mathematically equivalent to a linear combination of exponential time factors. This structure admits closed-form integration over arbitrary timestep intervals, allowing accurate latent updates with very few steps.

Concretely, for a sampling step from timestep $t_s$ to $t_e$ ($t_s > t_e$), we define the Analytic Transition Operator $\Phi$ as the latent displacement $ \Delta \mathbf{x}_{t_s \rightarrow t_e} $ induced by the velocity field $\mathbf{v}_\theta(\mathbf{x}_{t}, t)$. By analytically integrating this velocity based on Eq.~\eqref{eq:mixture_expectation} across timesteps, 
$\Phi(\mathbf{x}_{t_s}, t_s, t_e; \theta)$ admits the following closed-form expression:
\begin{equation}\small
\begin{aligned}
\Phi(\mathbf{x}_{t_s}, t_s, t_e; \theta)
&\triangleq \Delta \mathbf{x}_{t_s \rightarrow t_e} \\
&= \sum_{k=1}^K
\pi_k(\mathbf{x}_{t_s})\,
\mathbf{v}_k(\mathbf{x}_{t_s})\,
\mathcal{C}\!\left(\gamma_k(\mathbf{x}_{t_s}),\, t_s,\, t_e\right).
\end{aligned}
\label{eq:analytic_transition_main}
\end{equation}
where the Momentum Integral Coefficient $\mathcal{C}(\cdot)$ is defined as
\begin{equation}\small
\mathcal{C}(\gamma, t_s, t_e) =
\begin{cases}
\dfrac{\gamma^{1-t_e} - \gamma^{1-t_s}}{\ln \gamma}, & \gamma \neq 1, \\[6pt]
t_s - t_e, & \gamma = 1 ,
\end{cases}
\label{eq:momentum_coeff}
\end{equation}
The full derivation of Eq.~\eqref{eq:analytic_transition_main} and Eq.~\eqref{eq:momentum_coeff} is provided in \cref{appx:analytic}.

Crucially, the coefficient $\mathcal{C}(\gamma, t_s, t_e)$ smoothly reduces to the linear form $t_s - t_e$ as $\gamma \to 1$ (see \cref{appx:analytic}). This ensures numerical stability of our solver at the singularity, showing that our parameterization seamlessly bridges non-linear dynamics ($\gamma \neq 1$) and the linear flow regime ($\gamma = 1$).

Consequently, for any arbitrary step from $t_s$ to $t_e$, the next latent is given explicitly by $\mathbf{x}_{t_e} = \mathbf{x}_{t_s} - \Phi(\mathbf{x}_{t_s}, t_s, t_e; \theta)$,
allowing direct integration to target states. 
As shown in \cref{fig:framework}(b), while previous few-step students rely on very few straight sub-lines to fit a multi-step teacher trajectory whose tangent directions rapidly changing, leading to a poor approximation, ArcFlow trajectory from our analytic solver can naturally inherit the non-linearity and better align with the teacher’s overall trajectory.

\subsection{Flow Distillation with Analytic Solvers}
\label{method3}
Since ArcFlow naturally aligns with the teacher trajectory, we propose a practical flow distillation strategy based on a pre-trained teacher model.
As directly synthesizing the trajectory is infeasible within the flow matching framework which only predicts the velocity field, we reconstruct the teacher's trajectory by aligning its tangent direction (instantaneous velocity) at every timestep. 
For ArcFlow, this tangent is analytically derived via Eq.~\eqref{eq:mixture_expectation}, so the trajectory alignment reduces to a velocity-matching objective: minimizing the discrepancy between student and teacher instantaneous velocities at the sampled $(\mathbf{x}_{t}, t)$ pairs.

As shown in \cref{alg:ArcFlow_distill}, we propose a flow distillation method. For each timestep interval $[t_\mathrm{dst},t_\mathrm{src}]$, we train ArcFlow within this interval by iterating two steps as follows. 

\noindent\textbf{Mixed Latent Integration.} To enable the student to learn the teacher's velocity field over the whole interval $[t_{\mathrm{dst}},t_{\mathrm{src}}]$, we sample $n$ intermediate timesteps $\{t_1,\dots,t_n\}$ and construct the corresponding latents $\{\mathbf{x}_{t_i}\}$.  
Let $t_{\mathrm{src}}=t_0$ and $t_{\mathrm{dst}}=t_{n+1}$, each target latent $\mathbf{x}_{t_{i+1}}$ is then sequentially obtained by integrating over each sub-interval $[t_i,t_{i+1}]$.  
Within each $[t_i,t_{i+1}]$, we apply a mixed integration curriculum as training progresses: early training mainly follows teacher guidance to keep latents on the teacher manifold, while the student progressively takes over to gain the self-correction ability on its own generated latents.

Concretely, for each sub-interval $[t_i, t_{i+1}]$, we introduce a switching timestep $t_{\mathrm{mix}} = t_i - (1-\lambda)(t_i - t_{i+1})$,
where $\lambda$ is gradually increased from $0$ to $1$ during training.
Starting from $\mathbf{x}_{t_i}$, the teacher integrates the latent from $t_i$ to $t_{\mathrm{mix}}$, after which the student completes the integration to $t_{i+1}$.
This sequential handoff yields:
\begin{equation}\small
\mathbf{x}_{t_{i+1}} = \mathbf{x}_{t_i}
+ \int_{t_{\mathrm{mix}}}^{t_i} \mathbf{u}(\mathbf{x}_{t_i}, t_i)\,dt
+ \int_{t_{i+1}}^{t_{\mathrm{mix}}} \mathbf{v}(\mathbf{x}_t, t;\Theta)\,dt.
\label{eq:sequential_rollout}
\end{equation}
Here, $\mathbf{u}(\mathbf{x}_{t_i},t_i)$ is the instantaneous velocity predicted by the teacher at $(\mathbf{x}_{t_i},t_i)$; $\mathbf{v}(\mathbf{x}_t,t;\Theta)$ represents the velocity derived from momentum parameters $\Theta$ predicted by ArcFlow at $t_0$. Implementation details of the mixed integration are provided in the \cref{appx:mixed_trajectory}.

With every target latent state $\mathbf{x}_{t_i}$ obtained, we detach it from the computation graph and use it as the anchor for the velocity alignment step detailed below.

\noindent\textbf{Instantaneous Velocity Matching.}
At each $\mathbf{x}_{t_i}$, we proceed to align the velocity field predicted by the student with that of the teacher.
We compute the instantaneous velocity $\mathbf{v}(\mathbf{x}_{t_i}, {t_i}; \Theta) $ predicted by the student parameter $\Theta$ via Eq.~\eqref{eq:mixture_expectation}, and obtain the target velocity $\mathbf{u}(\mathbf{x}_{t_i}, {t_i})$ by evaluating the teacher network.
The optimization objective:
\begin{equation}\small
    \mathcal{L}_{\text{distill}} = \mathbb{E}_{{t_i}, \mathbf{x}_{t_i}} \left[ \left\| \mathbf{v}(\mathbf{x}_{t_i}, {t_i}; \Theta) - \mathbf{u} (\mathbf{x}_{t_i}, {t_i})\right\|^2 \right],
    \label{eq:distill_loss}
\end{equation}
Enforcing this loss ensures that the student's overall continuous trajectory adheres to the teacher's complex trajectory.

ArcFlow further simplifies this distillation process due to our momentum parameterization. 
As the momentum parameterization naturally inherits the non-linearity, matching the instantaneous velocity with very few timesteps ($n=2\sim4$) is sufficient for ArcFlow to learn the velocity field of the teacher, leading to high-precision restoration of the teacher trajectory and fast training process. Moreover, reduced distillation difficulty results in requiring fewer trainable parameters.  
While linear methods force the student to override the teacher's priors to fit linear rectification, which requires invasive full-parameter finetuning of large pre-trained models, ArcFlow naturally adapts to the non-linear trajectory. 
Empirically, we find that training only LoRA adapters on few layers and the output projection head is sufficient for convergence, which proves our assumption that ArcFlow enables efficient alignment with the teacher trajectory. 

\begin{table*}[t!]
\caption{Quantitative comparisons on Geneval, DPG-Bench and OneIG-Bench.$\dagger$ means the results are cited from pi-Flow~\cite{piflow} and TwinFlow~\cite{cheng2025twinflowrealizingonestepgeneration}. The NFE of Qwen-Image-20B is recorded as $50 \times2$ since it uses CFG~\cite{ho2022classifierfreediffusionguidance}.
}
\vspace{-0.1in}
\begin{center}
\renewcommand\arraystretch{1.1}
\scalebox{0.85}{
\begin{tabular}{lcccccccc}
\toprule
\multirow{2}{*}{Model}          & \multirow{2}{*}{NFE$\downarrow$}                                    & \multirow{2}{*}{Geneval$\uparrow$}                              & \multirow{2}{*}{DPG-Bench$\uparrow$}                                  & \multicolumn{5}{c}{OneIG-Bench}  \\ 
\cmidrule(lr){5-9}
& & & &
Alignment$\uparrow$ &
Text$\uparrow$ &
Diversity$\uparrow$ &
Style$\uparrow$ &
Reasoning$\uparrow$ \\

\midrule
FLUX.1-dev~\cite{flux2024}  & 50 & 0.66 & 84.16   &0.790$^\dagger$ &0.556$^\dagger$ &0.238$^\dagger$ &0.307$^\dagger$ &0.257$^\dagger$ \\ \midrule
SenseFlow (FLUX)~\cite{ge2025senseflowscalingdistributionmatching} & 2 &\underline{0.60} & 79.86&0.743 &\underline{0.230} &0.139 &\underline{0.341} &0.212 \\ 
Pi-Flow (GM-FLUX)~\cite{piflow} & 2&0.58 & \underline{82.36}&\underline{0.764} &0.141 &\textbf{0.216} &0.332 &0.212\\
ArcFlow-FLUX (Ours)      & 2 & \textbf{0.65} &\textbf{84.29} &\textbf{0.798} &\textbf{0.368} &\underline{0.210} &\textbf{0.350} &\textbf{0.224} \\ \midrule
Qwen-Image-20B~\cite{wu2025qwenimagetechnicalreport} & 50 $\times$ 2 & 0.87$^\dagger$ & 88.32$^\dagger$ &0.880$^\dagger$ &0.888$^\dagger$ &0.194$^\dagger$ &0.427$^\dagger$  & 0.306$^\dagger$ \\ \midrule
Qwen-Image-Lightning~\cite{qwen_image_lightning}    & 2  & \textbf{0.85} &\underline{88.42} &\underline{0.875} &\textbf{0.879} &0.098 &\underline{0.415} &\textbf{0.292} \\
pi-Flow (GM-Qwen)~\cite{piflow}    & 2 & 0.83 & 86.45& 0.837&0.634 &\underline{0.176} &0.382 &0.259 \\
TwinFlow (Qwen)~\cite{cheng2025twinflowrealizingonestepgeneration} &2&0.82 &87.01 & 0.862 &0.825 &0.130 &0.364 &0.267 \\ 
ArcFlow-Qwen (Ours)& 2  & \textbf{0.85} & \textbf{88.46} &\textbf{0.877} &\underline{0.853} &\textbf{0.182} &\textbf{0.421} &\underline{0.289} \\ \bottomrule
\end{tabular}
}
\end{center}
\label{table:quantitative_comparisons}
\end{table*}

\begin{table}[t!]
\caption{Quantitative comparisons on Align5000. FIDs and pFIDs are calculated against 50-step teacher generations. }
\vspace{-0.1in}
\begin{center}
\renewcommand\arraystretch{1.1}
\scalebox{0.85}{
\begin{tabular}{lcccc}
\toprule
Model & NFE$\downarrow$  & FID$\downarrow$ & pFID$\downarrow$ & CLIP$\uparrow$ \\
\midrule
FLUX.1-dev  & 50 & - & - &0.312  \\ \midrule
SenseFlow (FLUX) & 2 & \underline{27.55} &\textbf{9.25} &0.311 \\ 
Pi-Flow (GM-FLUX) & 2& 32.62 & 37.84 &\underline{0.314} \\
ArcFlow-FLUX (Ours)      & 2 & \textbf{16.83} &\underline{11.20} &\textbf{0.315} \\ \midrule
Qwen-Image-20B & 50 $\times$ 2 &- &- &0.325 \\ \midrule
Qwen-Image-Lightning    & 2  & 16.86 & 11.32 &0.320 \\
pi-Flow (GM-Qwen)    & 2 & 20.07 &12.42 &\underline{0.323} \\
TwinFlow (Qwen) &2 & \underline{16.77} & \underline{4.34} &0.320 \\ 
ArcFlow-Qwen (Ours)& 2 & \textbf{12.40} & \textbf{3.78} &\textbf{0.325} \\ \bottomrule
\end{tabular}
}
\end{center}
\label{table:quantitative_comparisons_2}
\end{table}

\begin{figure*}[t]
\begin{center}
\includegraphics[width=1\linewidth]{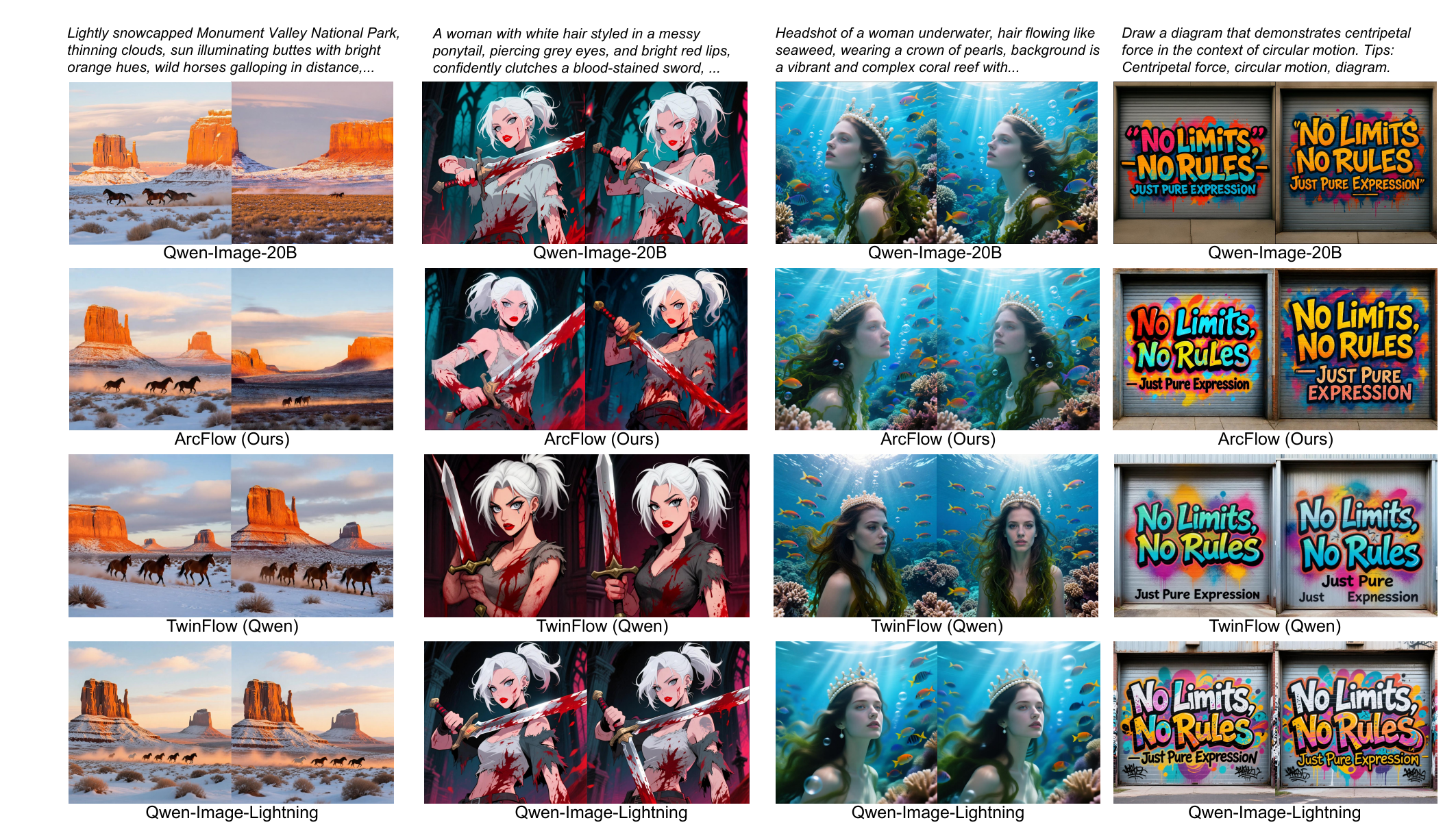}
\end{center}
\vspace{-0.3cm}
   \caption{Qualitative comparisons with methods distilled on Qwen-Image-20B (2NFE). Every column contains two images which are generated from the same batch of initial noise. 
   ArcFlow generates diverse samples that better align with teacher than competitors. 
   }
\label{fig:comparison_study}
\vspace{-0.40cm}
\end{figure*}

\begin{figure*}
\begin{center}
\includegraphics[width=0.85\linewidth]{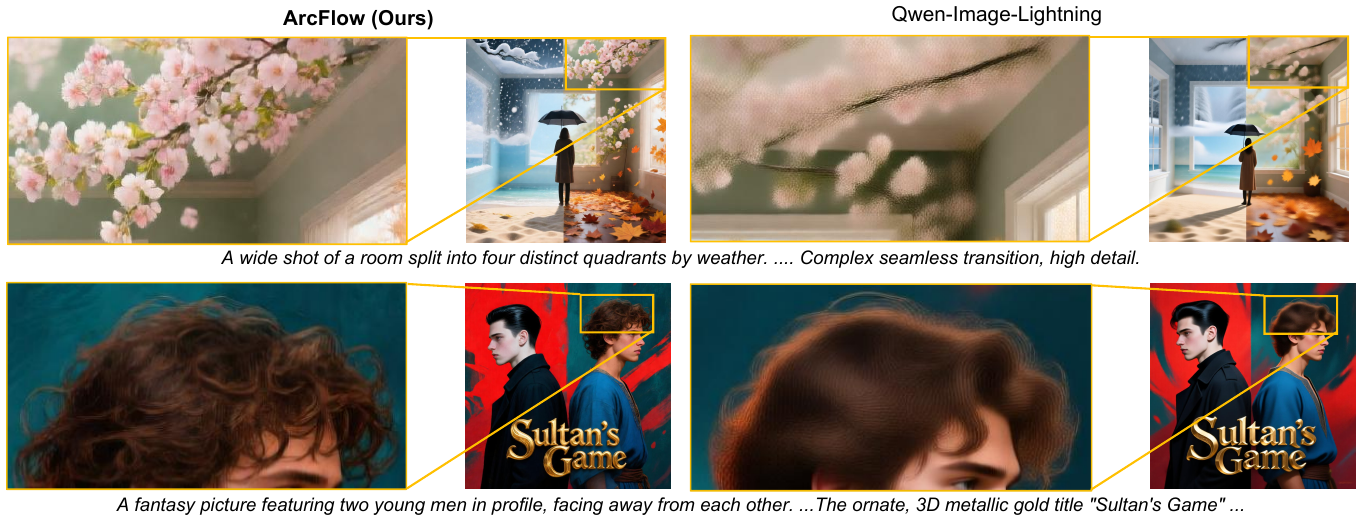}
\end{center}
\vspace{-0.3cm}
   \caption{Qualitative comparisons with Qwen-Image-Lightning. Our ArcFlow exhibits visibly clearer details. }
\label{fig:zoomin_comparison}
\vspace{-0.40cm}
\end{figure*}

\section{Experiments}
\subsection{Implementation Details}
\label{implementation_details}
We apply our distillation framework to two text-to-image models: Qwen-Image-20B~\cite{wu2025qwenimagetechnicalreport} and FLUX.1-dev~\cite{flux2024}.
As a parameter-efficient strategy, we freeze the vast majority of the backbone and train only 256-rank LoRA adapters injected into the feed-forward layers along with the final output projection head to accommodate the momentum parameter predictions.
We train ArcFlow on a large-scale prompt dataset (2.3 million samples) introduced by pi-Flow~\cite{piflow}. 
We provide more training details in 
\cref{appx:training_detail}.

We conduct evaluation on $1024 \times 1024$ image generation from three distinct benchmarks: (1) Geneval~\cite{ghosh2023genevalobjectfocusedframeworkevaluating} (complex object combination), (2) DPG-Bench~\cite{hu2024ella} (dense and long prompts), (3) OneIG-Bench~\cite{chang2025oneigbenchomnidimensionalnuancedevaluation} (complex prompts from distinct aspects). 
We additionally collect another dataset with 5,000 prompts, referred to the Align5000, composed of 3,200 prompts from HPSv2 prompt set ~\cite{wu2023humanhpsv2} and 1,800 prompts randomly sampled from the COCO 2014 validation set~\cite{lin2015microsoftcococommonobjects}. This combination covers both diverse artistic styles (HPSv2) and natural image distributions (COCO), enabling a more comprehensive evaluation of teacher alignment and distributional fidelity. Regarding metrics, the FIDs and patch FIDs (pFIDs) are computed against the 50-step teacher model to evaluate students' alignment with the teacher, and the CLIP similarity score measures the prompt alignment ability. In the patch FID metric, patch size is set to 64, and stride is set to 128.

\subsection{Comparison Study}
\label{subsec:comparison_study}

\noindent\textbf{Quantitative Results. }
We compare with recent few-step generative models distilled from the same teacher.
For FLUX.1-dev~\cite{flux2024}, we compare against: SenseFlow~\cite{ge2025senseflowscalingdistributionmatching}, which uses DMD; pi-Flow (GM-FLUX)~\cite{piflow}, which approximates the linear step with policy. For Qwen-Image-20B~\cite{wu2025qwenimagetechnicalreport}, we compare with: Qwen-Image-Lightning~\cite{qwen_image_lightning} based on VSD; TwinFlow~\cite{cheng2025twinflowrealizingonestepgeneration} based on self-adversarial loss; pi-Flow (GM-Qwen).
All models are set to NFE=2.

We observe that ArcFlow consistently outperforms or remains competitive with state-of-the-art few-step models across the three distinct benchmarks, demonstrating robust alignment with complex instructions. Specifically, while adversarial-based methods (e.g., Qwen-Image-Lightning) suffer from mode collapse (losing diversity to improve semantic alignment), ArcFlow achieves a substantial +85.7$\%$ improvement in Diversity on OneIG-Bench, proving that our parameterization effectively preserves the teacher's pre-trained priors and the generative diveristy. Furthermore, as shown in Table~\ref{table:quantitative_comparisons_2}, ArcFlow achieves the lowest FID and pFID across both backbones, indicating that our method ensures a significantly more precise alignment with the teacher’s generation compared to other linear shortcut baselines. Notably, although Qwen-Image-Lightning achieves competitive prompt-following scores, its inferior FID highlights a trade-off where perceptual optimization compromises trajectory fidelity, whereas ArcFlow maintains high fidelity to the original distribution (see \cref{appx:discussion_lightning} for detailed discussion).

\noindent\textbf{Qualitative Results.}
We conduct a qualitative comparison between ArcFlow and prior state-of-the-art methods by generating images from the same batch of initialized noise and comparing them with the teacher outputs, as shown in \cref{fig:comparison_study}. Linear distillation methods, including TwinFlow and Qwen-Image-Lightning, exhibit clear mode collapse and quality degradation, often producing nearly identical samples. Moreover, TwinFlow suffers from degraded visual aesthetics (the 3rd column), while Qwen-Image-Lightning shows blurred textures (background in the 3rd column) and structural artifacts (bent or duplicated swords in the 2nd column). These failures reveal a fundamental limitation of linear-step distillation methods in comprehensively approximating the teacher trajectory.

In contrast, at the same batch of initialized noise, ArcFlow consistently preserves both high visual quality and generation diversity, producing results that are more closely aligned with the teacher. This proves its superiority in both high-quality generation and high-precision approximation of the teacher, ensuring generation diversity.

Although Qwen-Image-Lightning achieves competitive quantitative performance in the benchmarks, further zoomed-in comparisons in \cref{fig:zoomin_comparison} show that ArcFlow yields noticeably finer and more coherent details. We attribute this discrepancy to the training objective of Qwen-Image-Lightning, which may sacrifice fine-grained visual fidelity for better semantic alignment objective, underscoring the inherent challenge of linear-step trajectory approximation. More discussion is provided in \cref{appx:discussion_lightning}.

\noindent\textbf{Convergence Speed and Stability.}
To validate the convergence speed and training stability of ArcFlow, we respectively distill ArcFlow, pi-Flow and TwinFlow based on Qwen-Image-20B. For pi-Flow and TwinFlow, we conduct training as guided in their official codebase.
We utilize the same training dataset as depicted in \cref{implementation_details}, and train models with a batch size of 16. We use the FID of Align5000 as the evaluation metric to measure the alignment between students and the teacher, and we evaluate the model at an iteration interval of 500 training steps.
As shown in \cref{fig:convergence}, ArcFlow converges significantly faster and with more stable FID reduction compared to other models. This validates that ArcFlow can efficiently leverage the pre-trained teacher weights, requiring only minor adaptation to reach near-optimal alignment. In contrast, TwinFlow, which uses full-parameter training, must override the teacher’s pre-trained weights due to the geometric mismatch, leading to high-error initial parameter state and slow convergence.
Notably, ArcFlow surpasses the FID of Qwen-Image-Lightning after only 1,000 training steps, demonstrating its efficiency in distillation training. We further visualize this convergence process comparison in \cref{appx:fig_convergence_visualization} and provide detailed analysis in \cref{appx:convergence_visualization}, which demonstrates the ArcFlow's superiority in inheriting and adapting to the pre-trained teacher knowledge, leading to efficient high-precision distillation.

\begin{table}[t]
\centering
\begin{minipage}{0.31\linewidth}
  \centering
  \caption{Momentum factor $\gamma$.}
  \vspace{-0.05in}
  \begin{small}
    \begin{sc}
      \begin{tabular}{lcc}
        \toprule
        $\gamma$ settings & FID $\downarrow$\\
        \midrule
        $\gamma \equiv 1$ & 17.06 \\
        $\gamma$ fixed    & 14.77 \\
        $\gamma$ learnable & 14.56 \\
        \bottomrule
      \end{tabular}
    \end{sc}
  \end{small}
  \vspace{0.5em}
  \label{tab:gamma}
\end{minipage}
\hfill
\begin{minipage}{0.31\linewidth}
  \centering
  \caption{$(N_v, N_\gamma)$. $K$ set to 16. }
  \vspace{-0.05in}
  \begin{small}
    \begin{sc}
      \begin{tabular}{lcc}
        \toprule
        $(N_v, N_\gamma)$ & FID $\downarrow$ \\
        \midrule
        $(K, 1)$ & 15.08 \\
        $(1, K)$ & 14.97 \\
        $(K, K)$ & 14.56 \\
        \bottomrule
      \end{tabular}
    \end{sc}
  \end{small}
  \label{tab:couples}
\end{minipage}
\hfill
\begin{minipage}{0.31\linewidth}
  \centering
  \caption{Numbers of momentum modes $K$.}
  \vspace{-0.05in}
  \begin{small}
    \begin{sc}
      \begin{tabular}{lcc}
        \toprule
        $K$ & FID $\downarrow$ & pFID $\downarrow$ \\
        \midrule
        8  & 12.54 & 4.17 \\
        16 & 12.4 & 3.78 \\
        32 & 12.39 & 3.69 \\
        \bottomrule
      \end{tabular}
    \end{sc}
  \end{small}
  \label{tab:ablation_k}
\end{minipage}
\vskip -0.2in
\end{table}

\begin{figure}
\begin{center}
\includegraphics[width=1\linewidth]{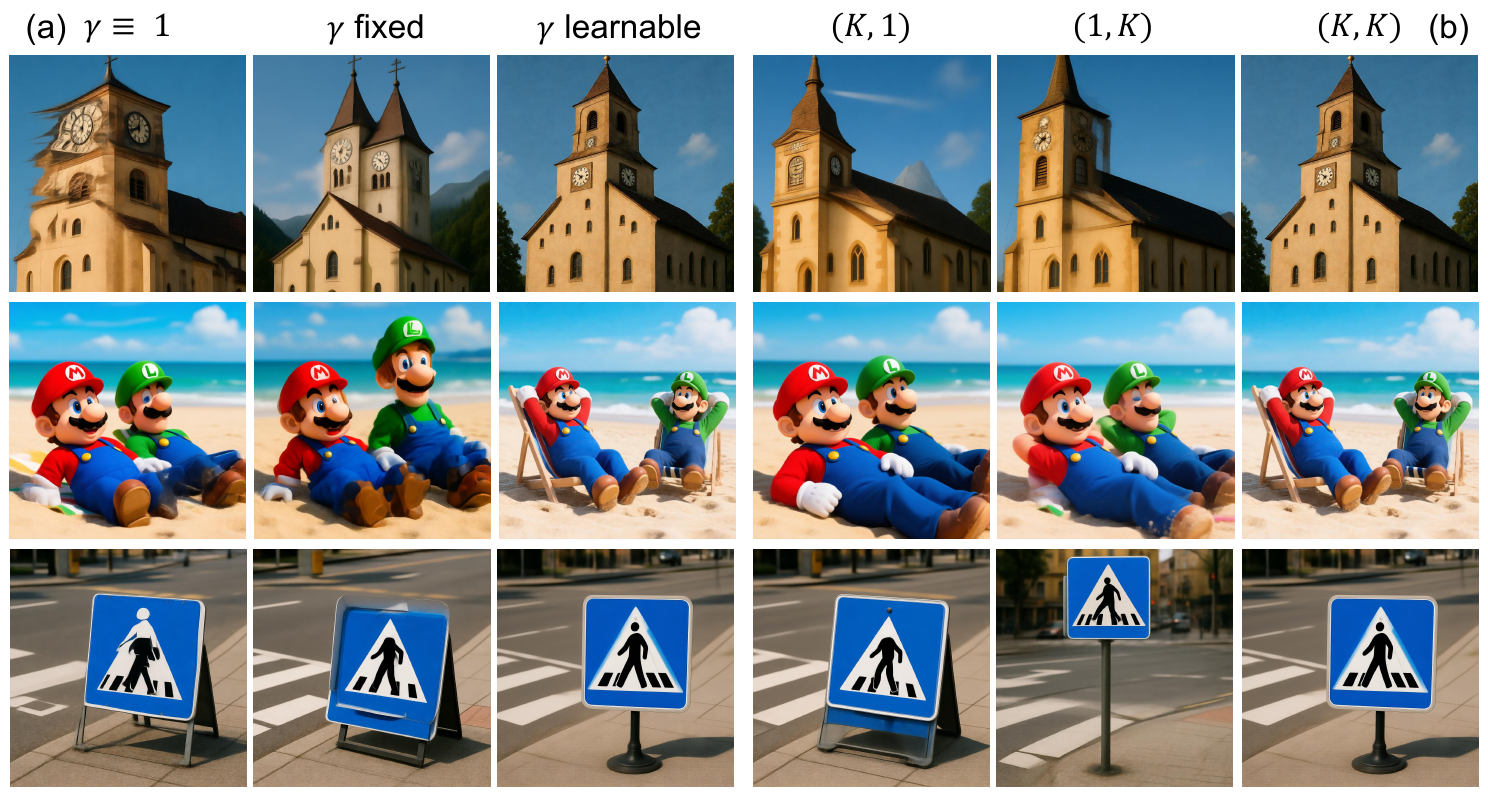}
\end{center}
\vspace{-0.4cm}
   \caption{Ablation qualitative results on core settings. 
   (a) Comparison of different momentum settings $\gamma$. (b) Comparison of different mixture configurations $(N_v, N_\gamma)$.
   }
\label{fig:ablation_comparison}
\end{figure}

\subsection{Ablation Study}

\noindent\textbf{Impact of Momentum Dynamics $\gamma$.}
We investigate the necessity of the momentum factor $\gamma$ in approximating trajectory tangent variation. All experiments are conducted on the Align5000 prompt set with 1,500 training steps.
As shown in \cref{tab:gamma}, setting $\gamma \equiv 1$ removes the explicit momentum factor from the parameterization.
In this setting, the model must rely solely on the velocity mixture to approximate the trajectory evolution, which forces the predicted velocity to implicitly compensate for the overall dynamics information, resulting in suboptimal alignment and inferior FID scores.
Then, introducing fixed momentum factors brings non-linearity into the trajectory and yields consistent improvements, demonstrating the benefit of explicit employment of non-linear trajectories.
Furthermore, making $\gamma$ learnable leads to the best performance, suggesting that adaptive momentum factors better capture the varying trajectory behaviors across different samples and timesteps.
\cref{fig:ablation_comparison}(a) highlights the importance of adaptively employing momentum for precise teacher–student alignment.

\noindent\textbf{Decoupling Velocity and Momentum Mixtures.}
Given the necessity of adaptive non-linearity, we further examine how each mixture component should be parameterized. We denote the configuration as $(N_v, N_\gamma)$, representing the number of independent basic velocities and momentum factors used across the $K$ mixture modes.
We compare our default $(N_v, N_\gamma)=(K,K)$ against two restricted variants: $(K,1)$, which forces diverse velocity directions to evolve under a unified motion pattern, and $(1,K)$, which restricts diverse momentum dynamics to start from the same basic velocities.
We train 1,500 steps for each experiment. \cref{tab:couples} and \cref{fig:ablation_comparison}(b) show that neither restricted setting is competitive with the performance of $(K,K)$.
This confirms that decoupling velocity and momentum clarifies the optimization task, while constraining either factor forces the remaining parameters to implicitly compensate for the missing dynamics, creating an overloaded and ambiguous learning target.

\noindent\textbf{Scalability of Mixture Components $K$.}
We study the effect of the mixture size $K$ by evaluating ArcFlow with $K \in \{8, 16, 32\}$, while keeping all other settings fixed.
As shown in \cref{tab:ablation_k}, increasing $K$ generally improves performance, indicating that a richer mixture enhances the model’s ability to capture non-linear trajectory tangent variation.
While $K=32$ achieves slightly better FID and pFID than $K=16$, the improvement is marginal.
Considering the diminishing returns, the increased parameter, and computational cost, we adopt $K=16$ as the default configuration, which offers a favorable trade-off between expressiveness and efficiency in the few-step distillation regime.
\section{Conclusion}
In this paper, we proposed ArcFlow, a few-step distillation framework that explicitly employs non-linear trajectories to approximate the complex dynamics of pre-trained diffusion teachers.
By parameterizing the velocity field as a mixture of continuous momentum processes, ArcFlow admits a closed-form analytic solver and enables accurate trajectory integration. We further introduced a flow distillation strategy to align the student’s analytical trajectory with the teacher.
Benefiting from its intrinsic non-linearity, ArcFlow ensures high-precision alignment with the teacher. Moreover, it avoids unstable adversarial objectives and invasive full-parameter training, leading to faster convergence and more efficient distillation. Extensive experiments demonstrated that ArcFlow consistently achieves superior generation quality with fewer trainable parameters compared to linear baselines.
We believe ArcFlow highlights the importance of respecting the underlying flow dynamics for efficient generative inference.
\section*{Impact Statement}

This paper aims to advance the efficiency of image generation models by enabling high-quality few-step inference through improved distillation techniques. Such progress can facilitate broader accessibility and deployment of generative models in practical applications, including creative tools, simulation, and content generation.

At the same time, as with prior work on image generation, our method could potentially be misused for generating misleading or harmful visual content. These concerns are not unique to our approach and are inherent to the broader class of generative image models. We emphasize that responsible deployment, including appropriate content moderation, usage policies, and the development of reliable AI-generated content detection mechanisms, remains important.

We believe that the technical contributions of this work primarily improve inference efficiency and fidelity, without introducing new ethical risks beyond those already present in existing diffusion-based image generation systems.

\bibliographystyle{plainnat}
\bibliography{main}
\clearpage
\newpage
\beginappendix
\begin{algorithm}[t]
\caption{Flow Distillation for ArcFlow}
\label{alg:ArcFlow_distill}
\begin{algorithmic}[1]
\REQUIRE NFE, Teacher $G_\psi$, Student $G_\phi$, Ratio $\lambda$
\STATE Sample start timestep $t_{\text{src}} \in \{ \frac{1}{\text{NFE}}, \dots, 1 \}$
\STATE Initialize $\mathbf{x}_{t_{\text{src}}}$ 
\STATE Sample timesteps $\{t_1, \dots, t_K\} \subseteq [t_{\text{src}} - \frac{1}{\text{NFE}}, t_{\text{src}}] $
\STATE $\Theta \leftarrow G_\phi(\mathbf{x}_{\text{src}}, t_{\text{src}})$;  
\FOR{$t \in \{t_1, \dots, t_K\}$}
    \STATE $\mathbf{x}_t \leftarrow \text{MixedIntegration}(\mathbf{x}_{\text{src}}, t_{\text{src}}, t; {\Theta}, \lambda)$
    \STATE $\hat{\mathbf{x}_t} \leftarrow \text{stopgrad}(\mathbf{x}_t)$
    \STATE $\mathbf{v}_{\text{stu}} \leftarrow \mathbf{v}(\hat{\mathbf{x}_t}, t; \Theta)$ 
    \STATE $\mathbf{u} \leftarrow G_\psi(\hat{\mathbf{x}_t}, t)$    
    \STATE $\mathcal{L} \leftarrow \mathcal{L} + \| \mathbf{v}_{\text{stu}} - \mathbf{u} \|^2$
\ENDFOR

\STATE Update $\phi$ using $\nabla_\phi \mathcal{L}$
\end{algorithmic}
\end{algorithm}

\section{Preliminaries}
In this section, we first introduce the Flow Matching framework and the Probability Flow ODE (PF-ODE). 
We then discuss the numerical simulation of this ODE, highlighting the difference between multi-step integration solvers and few-step solvers via distillation.

\noindent\textbf{Flow Matching and Probability Flow ODE.}
Let $p(\mathbf{x}_0)$ denote the data distribution and $p(\mathbf{x}_1) \sim \mathcal{N}(\mathbf{0}, \mathbf{I})$ the noise distribution. Flow Matching~\cite{lipman2023flowmatchinggenerativemodeling} defines a probability trajectory that iteratively transforms $p(\mathbf{x}_1)$ to $p(\mathbf{x}_0)$ over timesteps $t \in [0, 1]$. This transformation is driven by the probability flow ODE, which is defined as:
\begin{equation}
    \dfrac{d\mathbf{x}_t}{dt} = \mathbf{u}^*(\mathbf{x}_t, t),
\end{equation}
Integrating this velocity field $\mathbf{u}^*(\cdot)$ over timestep $t \in [0,1]$ forms the flow trajectories from noise $\mathbf{x}_1$ to data $\mathbf{x}_0$, which we term as trajectory integration.

To construct this velocity field, we use the Conditional Flow Matching (CFM) ~\cite{lipman2023flowmatchinggenerativemodeling} formulation. We define a linear trajectory between a data sample $\mathbf{x}_0$ and noise $\mathbf{x}_1$ to derive the latent $\mathbf{x}_t$ at any timestep $t$:
\begin{equation}
    \mathbf{u}(\mathbf{x}_t | \mathbf{x}_0, \mathbf{x}_1) = \frac{d}{dt} \mathbf{x}_t = \mathbf{x}_1 - \mathbf{x}_0,
\end{equation}
Therefore, the practical objective is to train a flow-matching model to approximate the marginal velocity field $\mathbf{u}_t(\cdot)$ via a predicted velocity field $\mathbf{v}_\mathbf{\theta}(\mathbf{x}, t)$ parameterized by $\mathbf{\theta}$, through the minimization of the expectation over these conditional flow trajectories: 
\begin{equation}
    \mathcal{L}_{\text{FM}}(\mathbf{\theta}) = \mathbb{E}_{t, p(\mathbf{x}_0), p(\mathbf{x}_1)} \left[ \| \mathbf{v}_\mathbf{\theta}(\mathbf{x}_t, t) - \mathbf{u}(\mathbf{x}_t | \mathbf{x}_0, \mathbf{x}_1) \|^2 \right],
    \label{eq:fm_loss}
\end{equation}

Once optimized, numerical solvers are employed to integrate the learned PF-ODE $v_\theta$, formulated as $\dfrac{d\mathbf{x}_t}{dt} = \mathbf{v}_\mathbf{\theta}(\mathbf{x}_t, t)$, enabling deterministic data generation.

\noindent\textbf{ODE Sampling and Distillation.}
Inference requires integrating the PF-ODE in timesteps from $t=1$ to $t=0$, yielding the sample $\hat{\mathbf{x}}_0 = \mathbf{x}_1 + \int_{1}^{0} v_\theta(\mathbf{x}_\tau, \tau) d\tau$. Since this integral is analytically intractable, numerical solvers (e.g., Euler or Heun~\cite{butcher2016numerical}) are employed to approximate the integration by accumulating discrete velocities in multiple timesteps. To minimize the discretization error, this approximation typically involves $40\sim100$ function evaluations (NFEs). With step size $\Delta t$, an Euler solver iteratively updates $x_t$ with the formula as follows:
\begin{equation}
    \mathbf{x}_{t-\Delta t} = \mathbf{x}_t - v_\mathbf{\theta}(\mathbf{x}_t, t) \cdot \Delta t,
\end{equation}
However, this iterative process imposes a significant computational bottleneck. To mitigate this, knowledge distillation is widely adopted, where the student $v_\mathbf{\phi}$ is trained to emulate the behavior of the teacher $v_\mathbf{\theta}$ across multiple timesteps, allowing it to traverse the trajectory with fewer NFEs. The standard distillation objective is typically formulated as a regression problem, where the student minimizes the discrepancy between its predicted velocity field and a target signal $\mathbf{y}_t$ derived from the teacher:
\begin{equation}
    \mathcal{L}_{\text{KD}}(\phi) = \mathbb{E}_{t, \mathbf{x}_t} \left[ \| v_\phi(\mathbf{x}_t, t) - \mathbf{y}_t(\mathbf{x}_t, v_\theta) \|^2 \right],
    \label{eq:distill_general}
\end{equation}
where $\mathbf{y}_t$ refers to the supervision target from the teacher, which varies depending on the specific method. 

\section{Discussions on Qwen-Image-Lightning}
\label{appx:discussion_lightning}

Although Qwen-Image-Lightning achieves competitive performance on prompt-alignment benchmarks, its inferior FID and pFID indicate clear degradation in generation quality at the distribution level. In particular, lower FID typically reflects distorted local statistics and weakened high-frequency details, such as blurred textures and over-smoothed structures, which are not captured by prompt-alignment metrics that primarily emphasize semantic correctness and global visual saliency. As a result, this indicates that although Qwen-Image-Lightning can generate images that remain semantically consistent with the prompt, it deviate noticeably from the teacher’s generation distribution in fine-grained structure and detail. 

Moreover, this discrepancy reveals a fundamental difference in distillation efficiency and objective. Qwen-Image-Lightning focuses on achieving perceptual and semantic alignment under limited steps, but does not explicitly enforce high-precision alignment with the teacher’s underlying generation trajectory, leading to a partial loss of the teacher’s original generative prior. As shown in \cref{appx:fig_compare2}, Qwen-Image-Lightning exihibits unstable performance in our test cases (especially in the first column), proving our assumption that it sacrifices the overall distribution correctness for higher semantic alignment. In contrast, ArcFlow directly distills the teacher’s non-linear velocity field and preserves its trajectory non-linearity through momentum-based parameterization. This design allows ArcFlow to rapidly converge to a high-precision teacher alignment with minimal training cost, thereby maintaining both strong prompt-following ability and high distributional fidelity in the few-step regime.

\section{Additional Technical Details}

\subsection{Mixed trajectory Integration}
\label{appx:mixed_trajectory}
This section provides the implementation details of the mixed latent integration strategy described in the main paper.

\noindent\textbf{Teacher Integration.}
For the teacher phase within each sub-interval $[t_i, t_{\mathrm{mix}}]$, we use the instantaneous velocity prediction
$\mathbf{u}(\mathbf{x}_{t_i}, t_i)$ as a constant velocity.
Since each sub-interval is sufficiently small, this approximation significantly reduces computational cost without affecting training stability.

\noindent\textbf{Student Integration via Analytic Transition.}
The student velocity field $\mathbf{v}(\mathbf{x}_t, t; \Theta)$ is induced by the momentum parameters $\Theta$ predicted by ArcFlow at $t_0$.
As ArcFlow defines a continuous velocity field, we apply the analytic transition operator $\Phi$ derived in Eq.~\eqref{eq:analytic_transition_main} to compute the integration from $t_{\mathrm{mix}}$ to $t_{i+1}$:
\begin{equation}
\begin{aligned}
\int_{t_{i+1}}^{t_{\mathrm{mix}}} \mathbf{v}(\mathbf{x}_t,t; \Theta)\,dt
&= \int_{t_{i+1}}^{t_0} \mathbf{v}(\mathbf{x}_t,t; \Theta)\,dt
 - \int_{t_{\mathrm{mix}}}^{t_0} \mathbf{v}(\mathbf{x}_t,t; \Theta)\,dt \\
&= \Phi(\mathbf{x}_{t_0}, t_0, t_{i+1}; \Theta)
 - \Phi(\mathbf{x}_{t_0}, t_0, t_{\mathrm{mix}}; \Theta).
\end{aligned}
\label{eq:student_analytic_step}
\end{equation}

This formulation allows efficient and exact integration of the student dynamics over arbitrary time intervals without numerical solvers.

\subsection{Momentum Factor Related Setting.}

\noindent\textbf{Log-parameterization of the Momentum Factor $\gamma$.}
According to the notion of momentum, the momentum factor $\gamma$ is required to be strictly positive.
However, directly regressing $\gamma$ introduces an implicit positivity constraint, which may lead to optimization difficulties and numerical instability during training, especially when the predicted values approach zero.
To address this, we parameterize the momentum factor in the logarithmic space.
Specifically, in training, the momentum factor projection head is designed to predict $\log \gamma$ instead of $\gamma$, and the actual momentum factor is recovered by exponentiation.
This reparameterization naturally enforces the positivity constraint while providing smoother gradients and more stable optimization behavior in practice.

\noindent\textbf{Projection Layer Initialization.}
To ensure that the model captures a diverse range of trajectory dynamics across timesteps, we initialize the momentum factors $\{\gamma_k\}_{k=1}^K$ as a geometric progression spanning the interval $[0.4, 5.0]$.
This range allows the mixture to cover both decelerating regimes ($\gamma < 1$) and accelerating regimes ($\gamma > 1$), enabling flexible modeling of complex flow trajectories.
In implementation, we first construct the geometric sequence in the $\gamma$ space and then convert it to the logarithmic domain.
Since the momentum factor projection layer is a linear layer, we initialize its weight matrix to zeros and assign the corresponding $\log \gamma_k$ values to the bias vector.
As a result, at the beginning of training, the predicted momentum factors exactly match the predefined geometric progression, providing a stable and interpretable initialization.

Crucially, we explicitly constrain one specific mode to be fixed at $\gamma = 1$.
This design introduces a linear inductive bias, serving as a stable anchor that allows the model to naturally fall back to standard linear flow matching when the velocity evolution is negligible.

\noindent\textbf{Learning Rate for the Momentum Factor Projection Layer.}
Since the momentum factor projection layer predicts $\log \gamma$ rather than $\gamma$, as shown in Eq.~\eqref{eq:mixture_expectation}, updates to this layer lead to exponential changes in the effective velocity field.
Consequently, using the same learning rate as other network components may cause unstable updates.

To improve numerical stability during training, we apply a reduced learning rate to this projection layer, specifically setting it to $0.1\times$ that of all other trainable layers.
This targeted adjustment effectively stabilizes optimization while preserving sufficient learning capacity for adapting the momentum factors.

\section{Training Details and Hyperparameters setting.}
\label{appx:training_detail}

\begin{table}[t!]
\caption{Detailed training configurations for distillation on Qwen-Image-20B and FLUX.1-dev. }
\vspace{-0.1in}
\begin{center}
\renewcommand\arraystretch{1.1}
\scalebox{0.85}{
\begin{tabular}{lcccc}
\toprule
Configuration & ArcFlow-Qwen & ArcFlow-FLUX \\\midrule
\textbf{Method-specific Settings} & & \\
Num of momentum modes $K$ & 16 & 16 \\
$\gamma$ initialization range &[0.5, 4.0] &[0.5, 4.0] \\
Num of intermediate timesteps &4 &4 \\
Trained NFEs &2 &2 \\
Mixed Trajectory Guidance Steps &1000 &2000 \\
\midrule
\textbf{Training Details} & & \\
Batch Size &384 &384 \\
Total Training Steps &7500 &8000 \\
\midrule
\textbf{Optimizer Settings} & & \\
Optimizer &AdamW & AdamW \\
Learning rate &$1e^{-4}$ &$1e^{-4}$ \\
Learning rate for $\gamma$ &$1e^{-5}$ &$1e^{-5}$ \\
Weight Decay &0 &0 \\
$(\beta_1,\beta_2)$ &(0.9, 0.95) &(0.9, 0.95) \\

\bottomrule
\end{tabular}
}
\end{center}
\label{table:training_configuration}
\end{table}

We freeze the teacher backbone and only train LoRA adapters together with extended output heads for predicting velocities, momentum factors, and gating probabilities.

For Qwen-Image-20B, we insert rank-256 LoRA adapters into a small subset of modules, including the image MLP, timestep embedding layers, and the text MLP blocks of the transformer. Specifically, LoRA is applied to the projection layers of the image MLP, both linear layers of the timestep embedder, and the text MLPs across all transformer blocks, while all remaining parameters are kept frozen.

For FLUX.1-dev, we apply rank-256 LoRA adapters to the projection and feed-forward modules that dominate the model’s conditional and feature transformation capacity. 
Specifically, LoRA is inserted into the MLP projection layers, the output projection head, the feed-forward networks in both the main and context branches, as well as the timestep embedding layers. 
All other parameters of the teacher backbone are kept frozen.

We conduct our experiment on 96 H100 GPUs. All models are trained with BF16 mixed precision. We detail our other training configurations in \cref{table:training_configuration}.

\section{Theoretical Analysis}
\subsection{Implementation and Derivation of the Analytic Transition Operator $\Phi$}
\label{appx:analytic}

This supplement provides the step-by-step derivation of Eq.~\eqref{eq:analytic_transition_main} and the limiting behavior $\gamma\to1$.

\noindent\textbf{Expansion of the velocity mixture.}
Assume the ArcFlow velocity at base state $\mathbf{x}_{t_s}$ is expressed as
\[
\mathbf{v}_\theta(\mathbf{x}_{t_s}, t)
= \sum_{k=1}^K \pi_k(\mathbf{x}_{t_s})\; \mathbf{v}_k(\mathbf{x}_{t_s})\; \gamma_k(\mathbf{x}_{t_s})^{\,1-t},
\]
where all mode-dependent quantities $\pi_k,\mathbf{v}_k,\gamma_k$ are evaluated at $\mathbf{x}_{t_s}$ and considered constant w.r.t. the integration variable $t$.

Then
\begin{align}
\Phi(\mathbf{x}_{t_s}, t_s, t_e;\theta)
&= \int_{t_e}^{t_s} \mathbf{v}_\theta(\mathbf{x}_{t_s}, t)\,dt \nonumber\\
&= \sum_{k=1}^K \pi_k(\mathbf{x}_{t_s})\,\mathbf{v}_k(\mathbf{x}_{t_s})
\int_{t_e}^{t_s} \gamma_k(\mathbf{x}_{t_s})^{\,1-t}\,dt,
\label{eq:supp_expand}
\end{align}

Define the scalar integral
\[
I(\gamma; t_s, t_e) \triangleq \int_{t_e}^{t_s} \gamma^{\,1-t}\,dt.
\]
For $\gamma\neq 1$ we compute
\begin{align}
I(\gamma; t_s, t_e)
&= \int_{t_e}^{t_s} e^{(1-t)\ln\gamma}\,dt
= \int_{t_e}^{t_s} e^{\ln\gamma}\,e^{-t\ln\gamma}\,dt \nonumber\\
&= \gamma \int_{t_e}^{t_s} e^{-t\ln\gamma}\,dt
= \gamma \cdot \frac{e^{-t\ln\gamma}}{-\ln\gamma}\Big|_{t_e}^{t_s} \nonumber\\
&= \frac{\gamma^{1-t_e}-\gamma^{1-t_s}}{\ln\gamma},
\label{eq:supp_I_non1}
\end{align}
This yields Eq.~\eqref{eq:momentum_coeff} for $\gamma\neq1$.

\noindent\textbf{Singular case $\gamma=1$ and continuity.}
When $\gamma=1$ the integrand equals $1$, hence $I(1;t_s,t_e)=t_s-t_e$.
We also show the limit $\lim_{\gamma\to1} I(\gamma;t_s,t_e)=t_s-t_e$ to prove continuity.
Set $\gamma=e^{h}$ with $h\to 0$. Then
\[
I(e^{h};t_s,t_e)=\frac{e^{h(1-t_e)}-e^{h(1-t_s)}}{h},
\]
As $h\to0$, apply Taylor expansion (or equivalently L'Hôpital's rule via $h$):
\[
\lim_{h\to0}\frac{e^{h(1-t_e)}-e^{h(1-t_s)}}{h}
= (1-t_e)-(1-t_s)=t_s-t_e,
\]
Thus the coefficient is continuous at $\gamma=1$, and the analytic expression recovers the linear dynamic mode.

\noindent\textbf{Full analytic operator.}
Combining Eq.\eqref{eq:supp_expand} and Eq.\eqref{eq:supp_I_non1} yields
\[
\Phi(\mathbf{x}_{t_s}, t_s, t_e;\theta)
= \sum_{k=1}^K \pi_k(\mathbf{x}_{t_s})\,\mathbf{v}_k(\mathbf{x}_{t_s})\,
\mathcal{C}(\gamma_k(\mathbf{x}_{t_s}), t_s, t_e),
\]
with $\mathcal{C}(\cdot)$ as in Eq.~\eqref{eq:momentum_coeff}.

\noindent\textbf{Numerical remarks.}
For numerical stability when $\gamma$ is very close to $1$, we branch to the second case ($t_s-t_e$) when $|\ln\gamma|<\epsilon$ ($\epsilon=10^{-6}$).

\subsection{Proof of Theorem \ref{thm:velocity_alignment}}
\label{proof:velocity_alignment}

In this section, we provide the proof for Theorem \ref{thm:velocity_alignment}. We demonstrate that the momentum parameterization in ArcFlow can accurately approximate any ground truth trajectory at $N$ discrete timesteps by using only $K=N$ momentum modes. 

\subsubsection{Problem Reformulation}

Let $\{t_1, \dots, t_N\}$ denote the set of sampled distinct timesteps, and let $\mathbf{u}^*_n = \mathbf{u}^*(\mathbf{y}, t_n) \in \mathbb{R}^D$ represent the corresponding ground truth velocities for any latent state $\mathbf{y}$ sampled from the data manifold.
Our objective is to demonstrate that there exists a parameter set $\theta = \{\pi_k, \mathbf{v}_k, \gamma_k\}_{k=1}^K$ that exactly satisfies the conditions:
\begin{equation}
    \sum_{k=1}^K \pi_k \mathbf{v}_k \gamma_k^{1-t_n} = \mathbf{u}^*_n, \quad \forall n \in \{1, \dots, N\},
    \label{eq:interpolation_target}
\end{equation}
where $K=N$. Directly solving Eq.~\eqref{eq:interpolation_target} is complicated by the bilinear coupling between $\pi_k$ and $\mathbf{v}_k$. Therefore, we introduce the composite parameter $\mathbf{w}_k \triangleq \pi_k \mathbf{v}_k$. The independence of the $D$ dimensions allows us to decouple the problem into $D$ identical scalar equations. Thus, we focus on a single scalar dimension, letting $w_k$ and $u^*_n$ denote the scalar components of $\mathbf{w}_k$ and $\mathbf{u}^*_n$, respectively.

Since our goal is to establish the existence of at least one feasible parameter set $\theta$, we may fix a subset of the parameters without loss of generality.
Specifically, we fix the momentum factors $\Gamma = \{\gamma_1, \dots, \gamma_N\}$ to be arbitrary distinct positive real values. With $\Gamma$ fixed, the exponential terms become known constants, and the problem of finding $\{\pi_k, \mathbf{v}_k\}$ reduces to solving for the composite weights $w_k$. This reduction holds because any valid solution for $w_k$ guarantees the existence of $\pi_k$ and $\mathbf{v}_k$.
Consequently, the problem can be formulated as the following linear system:
\begin{equation}
    \mathbf{M} \mathbf{c} = \mathbf{b},
    \label{eq:linear_system}
\end{equation}
where $\mathbf{c} = [w_1, \dots, w_K]^\top \in \mathbb{R}^K$ and $\mathbf{b} = [u^*_1, \dots, u^*_N]^\top \in \mathbb{R}^N$. The matrix $\mathbf{M} \in \mathbb{R}^{N \times K}$ is the basis matrix determined by the fixed $\gamma_k$:
\begin{equation}
    M_{nk} = \gamma_k^{1-t_n},
\end{equation}

Establishing the existence of $\theta$ for an arbitrary ground truth $\mathbf{b}$ is equivalent to guaranteeing that the linear system $\mathbf{M}\mathbf{c}=\mathbf{b}$ is solvable for any $\mathbf{b}$. This condition holds if and only if the basis matrix $\mathbf{M}$ is non-singular (invertible). Thus, the proof reduces to demonstrating the invertibility of $\mathbf{M}$.

\subsubsection{Proving Invertibility via Chebyshev Systems}

The solvability of the linear system relies on the non-singularity of the basis matrix $\mathbf{M}$. To establish this, we frame the problem within the theory of Chebyshev Systems. 

\begin{definition}[Chebyshev System]
\label{def:chebyshev}
Let $\{f_1, \dots, f_N\}$ be a set of continuous functions defined on an interval $\mathcal{I}$. This set constitutes a Chebyshev System if every non-trivial linear combination $F(t) = \sum_{k=1}^N c_k f_k(t)$ (where coefficients $c_k \in \mathbb{R}$ are not all simultaneously zero) possesses at most $N-1$ distinct zeros in $\mathcal{I}$.
\end{definition}

The significance of this definition lies in the Haar Condition, which directly links the zero-counting property of functions to the determinant of their basis matrix:

\begin{lemma}[Haar Condition \cite{cheney1966introduction}]
\label{lemma:haar}
If the set $\{f_1, \dots, f_N\}$ forms a Chebyshev System on $\mathcal{I}$, then for any set of distinct sampling points $\{t_1, \dots, t_N\} \subset \mathcal{I}$, the resulting matrix $\mathbf{\Phi}$ with entries $\Phi_{nk} = f_k(t_n)$ is non-singular.
\end{lemma}

To apply Lemma \ref{lemma:haar} to our specific problem, we must demonstrate that our proposed momentum dynamics functions satisfy the definition of a Chebyshev System.

\begin{proposition}
\label{prop:momentum_tsystem}
The set of functions $\{\gamma_k^{1-t}\}_{k=1}^N$, parameterized by distinct momentum factors $\gamma_k \in \mathbb{R}^+$, forms a Chebyshev System on $\mathbb{R}$.
\end{proposition}

\begin{proof}
Let $F(t)$ be an arbitrary non-trivial linear combination of the basis functions with coefficients $c_k \in \mathbb{R}$. We can rewrite the expression as a generalized polynomial of exponentials:
\[
F_N(t) = \sum_{k=1}^N c_k \gamma_k^{1-t} 
     = \sum_{k=1}^N (c_k \gamma_k) e^{-(\ln \gamma_k) t} 
     = \sum_{k=1}^N \alpha_k e^{\lambda_k t},
\]
Here, we define the new coefficients $\alpha_k \triangleq c_k \gamma_k$ (which remain non-zero if $c_k$ are non-zero) and the distinct exponents $\lambda_k \triangleq -\ln \gamma_k$. We prove that $F_N(t)$ has at most $N-1$ zeros by induction on $N$.

\textbf{Base case ($N=1$):} $F_1(t) = \alpha_1 e^{\lambda_1 t}$. Since exponentials are strictly positive and $\alpha_1 \neq 0$, $F_1(t)$ has no zeros.

\textbf{Inductive step:} Assume that any linear combination of $N-1$ exponentials $F_{N-1}(t)$ has at most $N-2$ distinct zeros. Suppose, for contradiction, that $F_N(t)$ has $N$ distinct zeros. Define the auxiliary function $G_N(t) = e^{-\lambda_1 t} F_N(t)$, which shares the same zeros as $F_N(t)$. Its derivative is:
\[
G_N'(t) = \frac{d}{dt} \left( \alpha_1 + \sum_{k=2}^N \alpha_k e^{(\lambda_k-\lambda_1)t} \right) = \sum_{k=2}^N \alpha_k (\lambda_k-\lambda_1) e^{(\lambda_k-\lambda_1)t}.
\]
Note that $G_N'(t)$ is a linear combination of $N-1$ exponentials with distinct exponents $\lambda_k - \lambda_1$. By the induction hypothesis, $G_N'(t)$ can have at most $N-2$ zeros. However, by Rolle's Theorem, if $G_N(t)$ has $N$ distinct zeros, its derivative $G_N'(t)$ must have at least $N-1$ distinct zeros. This contradiction implies that $F_N(t)$ cannot have $N$ distinct zeros.
\end{proof}
\textbf{Conclusion.}
Since Proposition \ref{prop:momentum_tsystem} confirms that our basis functions form a Chebyshev System, Lemma \ref{lemma:haar} ensures that the matrix $\mathbf{M}$ is invertible for any set of distinct timesteps. This guarantees the existence of a solution vector $\mathbf{c}$. 
Translating this mathematical result back to our original objective, the solvability of $\mathbf{c}$ implies that for any ground truth velocities $\mathbf{u}^*_n$, we can explicitly construct a parameter set $\theta$ (e.g., by setting $\pi_k=1$ and $\mathbf{v}_k$ from $\mathbf{c}$) that satisfies Eq.~\eqref{eq:interpolation_target} exactly. This completes the proof of Theorem \ref{thm:velocity_alignment}, theoretically validating that the proposed momentum parameterization possesses sufficient expressivity to perfectly align with arbitrary trajectory dynamics on the data manifold.

\section{More Experiment Results}
\subsection{Ablations on Mixed Trajectory Integration}
\label{appx:ablation_mixed_trajectory}

\begin{figure}[t]
    \centering
    \includegraphics[width=0.8\linewidth]{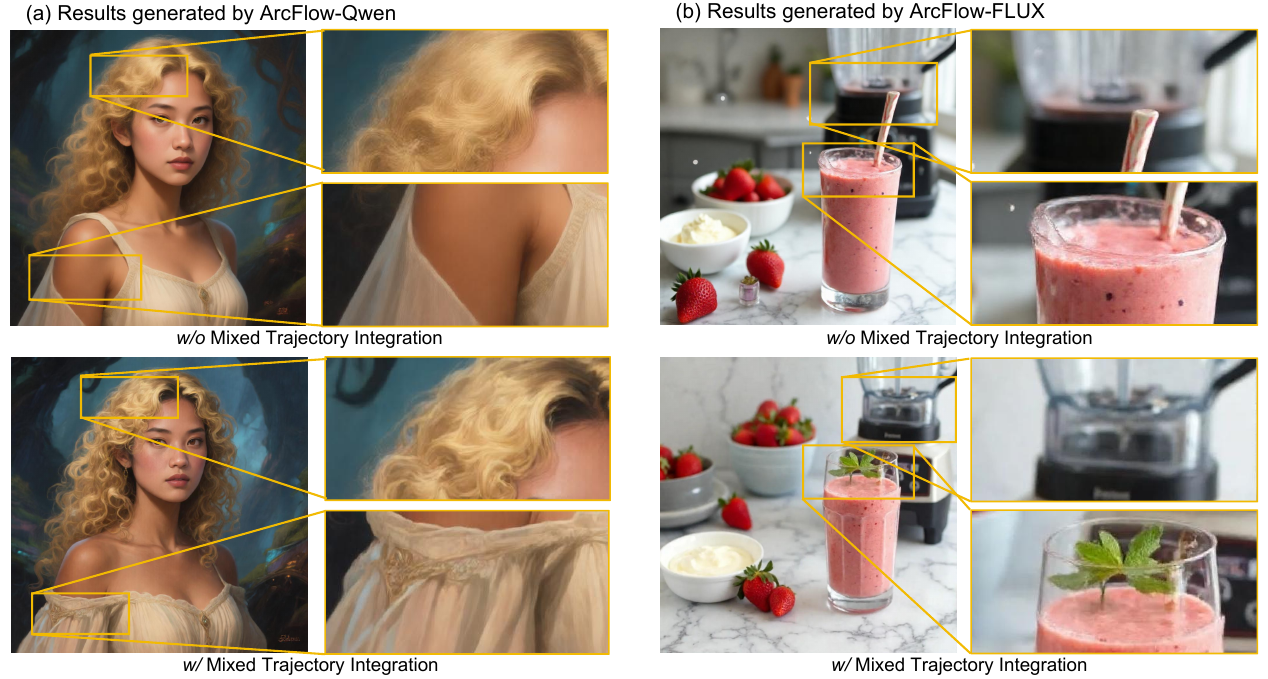}
    \vspace{-0.1cm}
    \caption{Visualization of the ablation effect on the adoption of mixed trajectory integration for training. Here, (a) is generated from distilling Qwen-Image-20B, and (b) is generated from distilling FLUX.1-dev. }
    \label{appx:fig_ablation_mixed}
    \vspace{-0.2cm}
\end{figure}

\begin{table}\small
\centering
\caption{Quantitative ablation results on the adoption of mixed trajectory integration for training. }
\label{tab:ablation_mixed_trajectory}
\begin{tabular}{lc}
\toprule
Method & FID$\downarrow$ \\
\midrule
\textbf{ArcFlow-Qwen} & \\
\textit{w/o} Mixed Trajectory Integration & 14.04 \\
\textit{w/} Mixed Trajectory Integration & 13.52 \\
\midrule
\textbf{ArcFlow-FLUX} & \\
\textit{w/o} Mixed Trajectory Integration & 19.17 \\
\textit{w/} Mixed Trajectory Integration & 18.21 \\\bottomrule
\end{tabular}
\end{table}

To further validate the effectiveness of the proposed mixed trajectory integration strategy during training, we conduct ablation studies by training ArcFlow models with and without mixed trajectory integration on both Qwen-Image-20B and FLUX.1-dev. In all settings, models are trained for 3,000 iterations with a batch size of 16, and evaluated using teacher-alignment FID on the Align5000 dataset. As shown in \cref{tab:ablation_mixed_trajectory}, adopting mixed trajectory integration consistently improves FID across both backbones, indicating more accurate alignment with the teacher distribution. We further provide qualitative comparisons in \cref{appx:fig_ablation_mixed}. Models trained with mixed trajectory integration produce images with richer local details and sharper structures, benefiting from learning the velocity field while staying closer to the teacher trajectory in early training. In contrast, models trained without this strategy, while preserving comparable global structure, exhibit smoother and less detailed results, suggesting that the student is more prone to learning inaccurate velocity estimates at early stages, which leads to slower and less stable convergence.

\subsection{Convergence Visualization}
\label{appx:convergence_visualization}
\begin{figure}[t]
    \centering
    \includegraphics[width=0.8\linewidth]{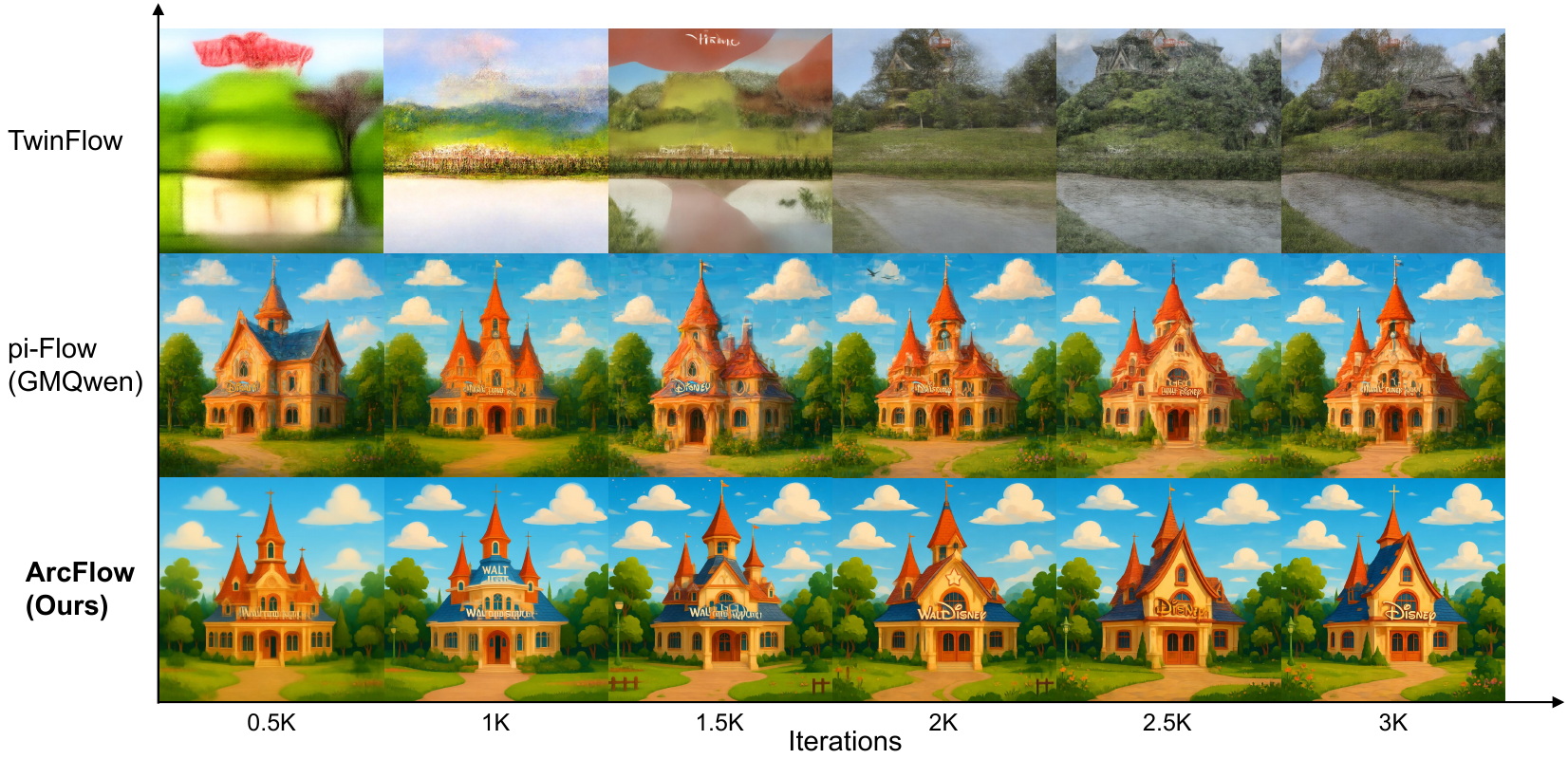}
    \vspace{-0.1cm}
    \caption{Convergence visualization across different student methods based on Qwen-Image-20B. }
    \label{appx:fig_convergence_visualization}
    \vspace{-0.2cm}
\end{figure}

To further validate ArcFlow's superior convergence speed and stability, we conducted a visualization experiments. We reuse the same training checkpoint as in the convergence analysis of \cref{subsec:comparison_study} to ensure a fair comparison. \cref{appx:fig_convergence_visualization} shows our visualization result generated from one single prompt. 

As shown in \cref{appx:fig_convergence_visualization}, ArcFlow already exhibits a coherent global structure after only 0.5K training iterations, with most stochastic artifacts and irregular noise largely suppressed. At this early stage, the generated images mainly suffer from mild over-smoothing, rather than structural corruption, indicating that the model has already entered a meaningful generative regime. This behavior suggests that ArcFlow can immediately benefit from its natural compatibility with the pre-trained teacher weights, enabling effective adaptation without requiring extensive retraining.
As training proceeds, the visual quality of ArcFlow improves in a stable and monotonic manner. By 3K iterations, the generated images reach a level where no obvious visual defects can be identified by human inspection, demonstrating both fast convergence and strong training stability.

In comparison, pi-Flow preserves reasonable global structure at early iterations; however, it consistently struggles with residual noise artifacts throughout the training process. Even with increased iterations, the presence of scattered noise prevents a clear improvement in perceptual quality, resulting in noticeably slower and less stable convergence.

Moreover, TwinFlow exhibits the weakest convergence behavior. As a linear distillation method, its parameterization conflicts with the teacher’s inherently complex trajectory, which prevents effective reuse of the teacher’s pre-trained weights at initialization. Consequently, TwinFlow is forced to re-learn a viable representation from a high-error state, leading to significantly slower convergence and inferior visual quality during early and intermediate training stages.

\subsection{Inference time}

\begin{table}\small
\centering
\caption{Average inference time of different models on the same prompt and $1024\times1024$ resolution, all generating with 2 NFEs.}
\label{tab:inference_time}
\begin{tabular}{lc|lc}
\toprule
\textbf{Qwen-Image Students} & Inference Time (s) & \textbf{FLUX Students} & Inference Time (s)\\
\midrule
Qwen-Image-Lightning & 1.718 & SenseFlow (FLUX) & 1.432 \\
TwinFlow & 1.372 & pi-Flow (GMFLUX) & 1.470\\
pi-Flow (GMQwen) & 1.440 &  \textbf{ArcFlow-FLUX (Ours)} &  1.466 \\
\textbf{ArcFlow-Qwen (Ours)} &  1.411 & &  \\
\bottomrule
\end{tabular}
\end{table}

To quantitatively validate our generation acceleration, we measure the inference time of our method and other few-step baselines by running each model five times on the same prompt and reporting the average.
Table~\ref{tab:inference_time} summarizes the inference time of different student models evaluated at a resolution of $1024\times1024$, with all methods generating images using 2 NFEs.
We observe that Qwen-Image-Lightning exhibits the longest inference time, which can be attributed to its use of multiple LoRA adapters, introducing additional low-rank computations during inference.
In contrast, TwinFlow and SenseFlow achieve the lowest inference time, as they are trained via full-parameter finetuning, where the adapted weights directly overwrite the original parameters and incur no additional computational overhead at inference time.
Our methods, ArcFlow-Qwen and ArcFlow-FLUX, fall between these two extremes. This indicates that the additional parameters introduced by our finetuning strategy incur only a negligible increase in floating-point operations, resulting in inference times comparable to fully finetuned baselines.
Overall, the results demonstrate that our approach achieves a favorable balance between generation quality and inference efficiency, without sacrificing the low-latency advantage crucial for few-step image generation.

\section{Limitations and Future Work}
\begin{figure}
    \centering
    \includegraphics[width=0.3\linewidth]{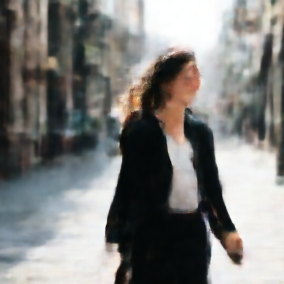}        
    \caption{One failure case of ArcFlow, as 1-NFE inference produces blurry results. }
    \label{fig:failure}
    \vspace{-0.5cm}
\end{figure}
\cref{fig:failure} illustrates a representative limitation of our method. Specifically, when forced to degenerate to the extreme setting of single-step inference (1 NFE), ArcFlow exhibits severe degradation in generation quality and fails to produce meaningful results. We attribute this limitation to the difficulty of accurately modeling the momentum factor $\gamma$ under a 1 NFE regime, where $\gamma$ becomes highly sensitive and challenging to predict without sufficient modeling capacity. A potential direction to address this issue is to design deeper or more expressive network layers dedicated to modeling $\gamma$. In addition, we plan to validate the effectiveness of our method across models with diverse parameter scales. This part is left as future work.

\section{Additional Qualitative Results}
\label{appx:qualitative_results}

\subsection{Comparison of Few-step Students on Qwen-Image-20B}
\begin{figure}[t]
    \centering
    \includegraphics[width=1.0\linewidth]{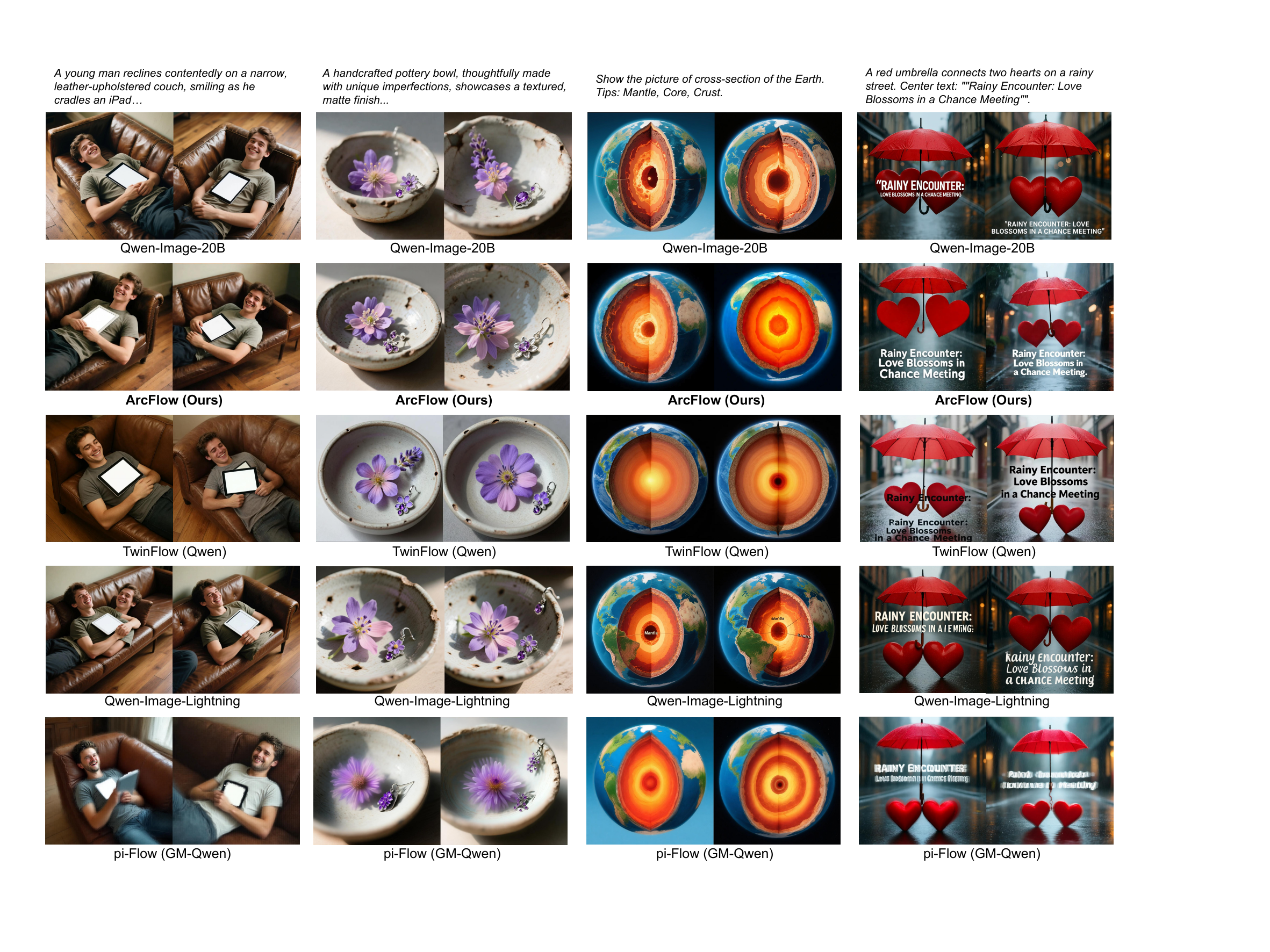}
    \caption{Additional qualitative comparisons between different student models distilled on Qwen-Image-20B. Note that results in each column are generated from the same batch of initial noise. }
    \label{appx:fig_compare2}
\end{figure}
We provide additional comparison results of student models that are based on Qwen-Image-20B in \cref{appx:fig_compare2}. 

\subsection{Comparison of Few-step Students on FLUX.1-dev}
\begin{figure}[t]
    \centering
    \includegraphics[width=0.9\linewidth]{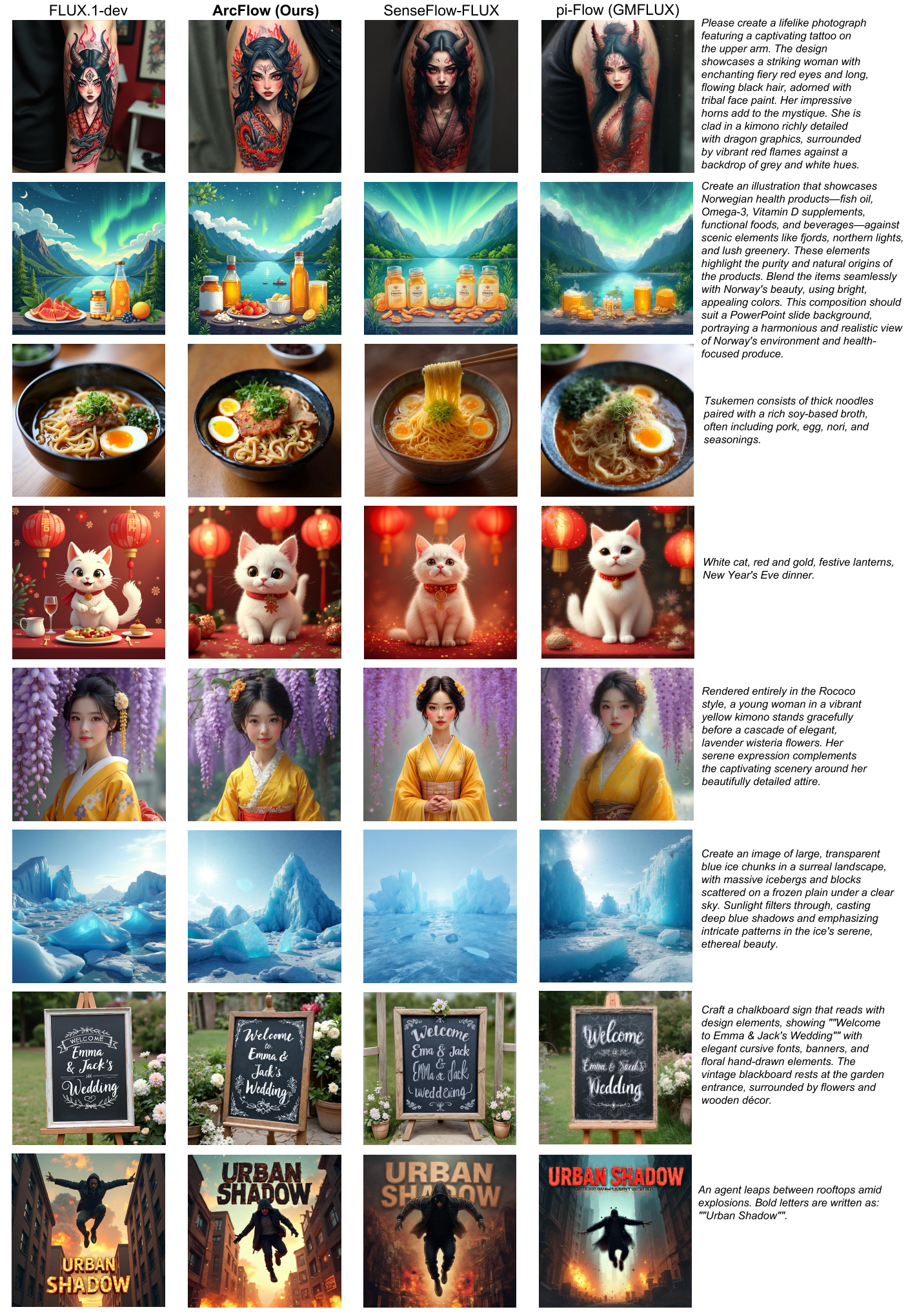}
    \caption{Qualitative comparisons between different student models distilled on FLUX.1-dev. }
    \label{appx:fig_compare_flux}
\end{figure}

We provide additional comparison results of student models that are based on FLUX.1-dev in \cref{appx:fig_compare_flux}. 

\subsection{More High-Resolution Visualizations}
In this section, we present additional qualitative examples generated by ArcFlow to further demonstrate its generative performance. All prompts are randomly sampled, and the results are shown directly without any manual selection or filtering. Visualization results are shown in \cref{appx:vis1}, \cref{appx:vis2}, \cref{appx:vis3}.

\begin{figure}[t]
    \centering
    \includegraphics[width=0.8\linewidth]{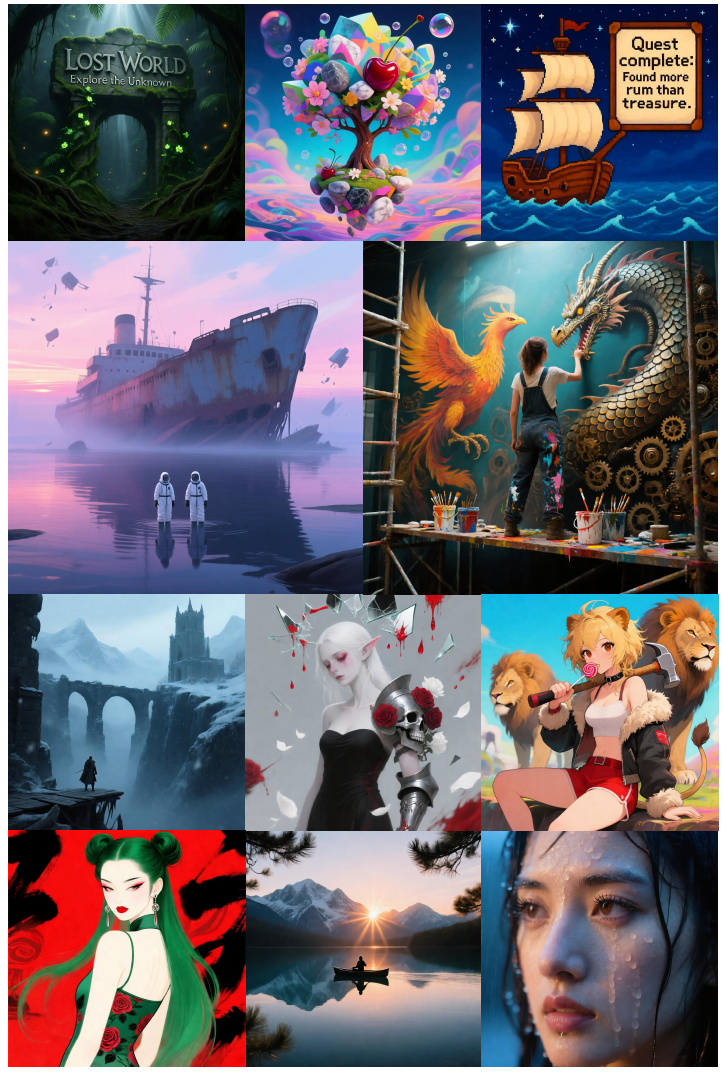}
    \caption{Visualization of ArcFlow-Qwen (NFE=2). Each image is of $1024\times1024$ resolution. }
    \label{appx:vis1}
\end{figure}

\begin{figure}[t]
    \centering
    \includegraphics[width=0.8\linewidth]{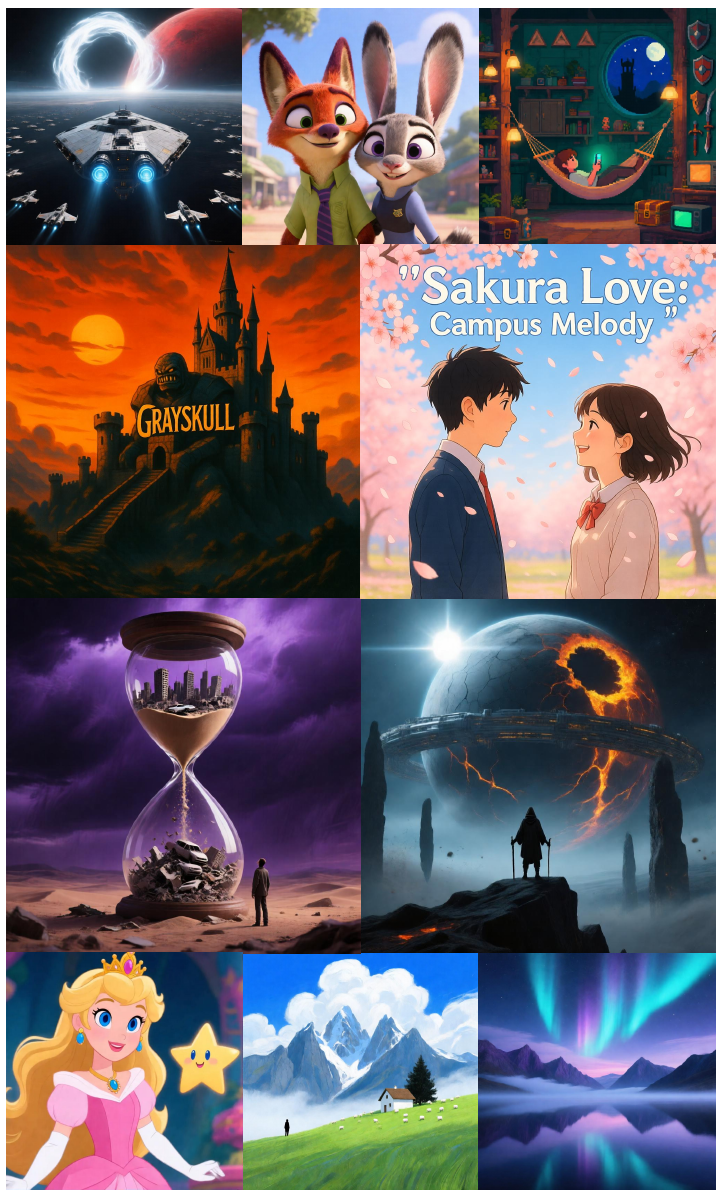}
    \caption{Visualization of ArcFlow-Qwen (NFE=2). Each image is of $1024\times1024$ resolution. }
    \label{appx:vis2}
\end{figure}

\begin{figure}[t]
    \centering
    \includegraphics[width=0.8\linewidth]{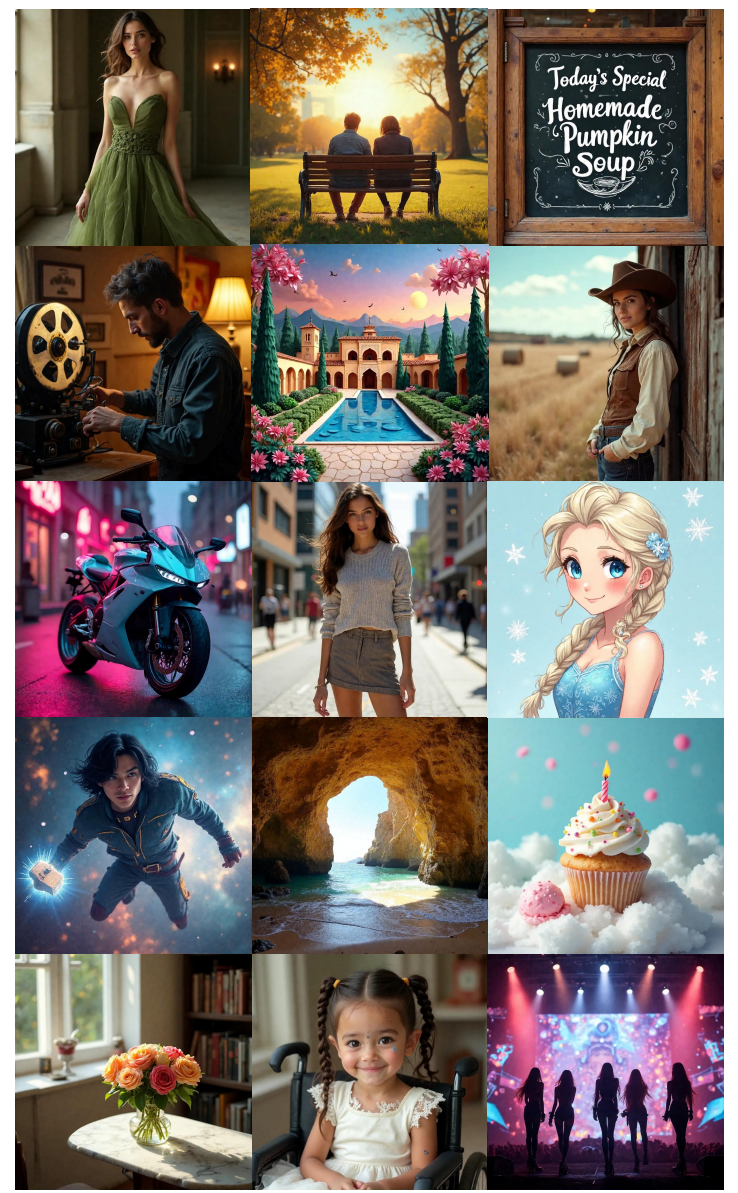}
    \caption{Visualization of ArcFlow-FLUX (NFE=2). Each image is of $1024\times1024$ resolution. }
    \label{appx:vis3}
\end{figure}
\end{document}